%% file: main.tex
\documentclass[10pt,journal,compsoc]{IEEEtran}

\ifCLASSOPTIONcompsoc
  \usepackage[nocompress]{cite}
\else
  \usepackage{cite}
\fi

\usepackage[table,xcdraw]{xcolor}
\usepackage[utf8]{inputenc} % allow utf-8 input
\usepackage[T1]{fontenc}    % use 8-bit T1 fonts
\usepackage{ragged2e}
\usepackage[colorlinks, bookmarks=false]{hyperref}       % hyperlinks
\usepackage{url}   % simple URL typesetting
\usepackage{booktabs}       % professional-quality tables

\usepackage{amsmath, amsthm, amsfonts, amssymb}
\usepackage{mathrsfs}
\usepackage{algorithm}
\usepackage{algorithmic}

\usepackage{multirow}
\usepackage{graphicx}
\usepackage{makecell}
\usepackage{caption}
\usepackage{tablefootnote}
\usepackage{wrapfig}
\usepackage{tikz, lscape}
\usetikzlibrary{matrix}

\usepackage[labelformat=simple]{subfig}        
  
\newcommand\figref[1]{figure~\ref{#1}}

\usepackage{cleveref}
\crefname{figure}{figure}{figure}
\crefname{subfig}{figure}{figure}
\crefname{table}{Table}{Table}
\crefname{equation}{Equation}{Equation}
\crefname{section}{Section}{Section}
\crefname{algorithm}{Algorithm}{Algorithm}

\newtheorem{assumption}{Assumption}
\newtheorem{theorem}{Theorem}
\newtheorem{lemma}{Lemma}
\newtheorem{proposition}{Proposition}

\def \E {\mathbb{E}}
\def \w {\boldsymbol{w}}
\def \m {\boldsymbol{m}}
\def \x {\boldsymbol{x}}

\def \xi {\boldsymbol{x}_i}
\def \yi {\boldsymbol{y}_i}

\usepackage[numbers]{natbib}
\begin{document}
%
% paper title
% Titles are generally capitalized except for words such as a, an, and, as,
% at, but, by, for, in, nor, of, on, or, the, to and up, which are usually
% not capitalized unless they are the first or last word of the title.
% Linebreaks \\ can be used within to get better formatting as desired.
% Do not put math or special symbols in the title.

\title{Systematic Investigation of Sparse Perturbed Sharpness-Aware Minimization Optimizer}

\author{Peng~Mi,~
        Li~Shen,~
        Tianhe~Ren,~
        Yiyi~Zhou,~
        Tianshuo~Xu,
        Xiaoshuai~Sun,
        Tongliang Liu,~\IEEEmembership{Senior Member,~IEEE,} 
        Rongrong~Ji,~\IEEEmembership{Senior Member,~IEEE,}~       Dacheng~Tao,~\IEEEmembership{Fellow,~IEEE}
        % <-this % stops a space
\IEEEcompsocitemizethanks{
\IEEEcompsocthanksitem P. Mi and T. Ren, Y. Zhou, T. Xu, X. Sun, and R. Ji are with the Key Laboratory of Multimedia Trusted Perception and Efficient Computing, Ministry of Education of China, Xiamen University, P.R. China. E-mails: mipeng@stu.xmu.edu.cn; rentianhe@stu.xmu.edu.cn; zhouyiyi@xmu.edu.cn; xutianshuo@stu.xmu.edu.cn; xssun@xmu.edu.cn; rrji@xmu.edu.cn.
\IEEEcompsocthanksitem L. Shen is with JD Explire Adacemy, Beijing, China. E-mail: mathshenli@gmail.com.
\IEEEcompsocthanksitem T. Liu and D. Tao are with The University of Sydney, Australia. E-mail:  tongliang.liu@sydney.edu.au; dacheng.tao@gmail.com
}
%\thanks{(Corresponding author: Li Shen.)}
\thanks{Manuscript revised June 21, 2023.}
}

% The paper headers
\markboth{IEEE TRANSACTIONS ON PATTERN ANALYSIS AND MACHINE INTELLIGENCE, VOL. XX, NO. XX}{}%

\input{texfile/abstract}

% make the title area
\maketitle

\IEEEdisplaynontitleabstractindextext
% \IEEEdisplaynontitleabstractindextext has no effect when using
% compsoc or transmag under a non-conference mode.

\IEEEpeerreviewmaketitle

\input{texfile/introduction}

\input{texfile/related_work}

\input{texfile/rethinking_sam}
\input{texfile/methodology}
\input{texfile/experiment}

\input{texfile/conclusion}

% \newpage
\bibliographystyle{IEEEtran}
\bibliography{ref}

\clearpage

\input{texfile/appendix}

\end{document}

%% file: texfile/abstract.tex
\IEEEtitleabstractindextext{
\begin{abstract}
\justifying
Deep neural networks often suffer from poor generalization due to complex and non-convex loss landscapes. Sharpness-Aware Minimization (SAM) is a popular solution that smooths the loss landscape by minimizing the maximized change of training loss when adding a perturbation to the weight. However, indiscriminate perturbation of SAM on all parameters is suboptimal and results in excessive computation, double the overhead of common optimizers like Stochastic Gradient Descent (SGD).
In this paper, we propose Sparse SAM (SSAM), an efficient and effective training scheme that achieves sparse perturbation by a binary mask. To obtain the sparse mask, we provide two solutions based on Fisher information and dynamic sparse training, respectively. We investigate the impact of different masks, including unstructured, structured, and $N$:$M$ structured patterns, as well as explicit and implicit forms of implementing sparse perturbation.
We theoretically prove that SSAM can converge at the same rate as SAM,~\emph{i.e.}, $O(\log T/\sqrt{T})$. Sparse SAM has the potential to accelerate training and smooth the loss landscape effectively. Extensive experimental results on CIFAR and ImageNet-1K confirm that our method is superior to SAM in terms of efficiency, and the performance is preserved or even improved with a perturbation of merely 50\% sparsity.
Code is available at \url{https://github.com/Mi-Peng/Systematic-Investigation-of-Sparse-Perturbed-Sharpness-Aware-Minimization-Optimizer}.
\end{abstract}
}

%% file: texfile/introduction.tex
\IEEEraisesectionheading{\section{Introduction}
\label[section]{sec:intro}}
\IEEEPARstart{O}{ver} the past decade or so, the great success of deep learning has been due in great part to ever-larger model parameter sizes~\cite{vit, wideresnet, transformer, bert, swin-transformer, biggan}.
However, the excessive parameters also make the model inclined to poor generalization.
To overcome this problem, numerous efforts have been devoted to  training algorithm~\cite{early-stop, entropy-regularization, label-smoothing}, data augmentation~\cite{cutout, mixup, cutmix}, and network design~\cite{dropout, bn}. 

One important finding in recent research is the connection between the geometry of loss landscape and model generalization~\cite{LBtrain, sam, flat-minima, exploring-generalization,adversarial-robust}. In general, the loss landscape of the model is complex and non-convex, which makes model tend to converge to sharp minima. Recent endeavors~\cite{LBtrain, flat-minima, exploring-generalization} show that the flatter the minima of convergence, the better the model generalization. This discovery reveals the nature of previous approaches~\cite{early-stop, dropout, cutout, cutmix, mixup, bn} to improve generalization,~\emph{i.e.}, smoothing the loss landscape. 

Based on this finding, Foret~\emph{et al.}~\cite{sam} propose a novel approach to improve model generalization called \emph{sharpness-aware minimization} (SAM), which simultaneously minimizes loss value and loss sharpness. SAM quantifies the landscape sharpness as the maximized difference of loss when a perturbation is added to the weight. When the model reaches a sharp area, the perturbed gradients in SAM help the model jump out of the sharp minima. In practice, SAM requires two forward-backward computations for each optimization step, where the first computation is to obtain the perturbation and the second one is for parameter update. Despite the remarkable performance~\cite{sam,asam,esam,vit-sam}, This property makes SAM double the computational cost of the conventional optimizer,~\emph{e.g.}, SGD~\cite{bottou2010large}.
% and Adam~\cite{adam}. 

Since SAM calculates perturbations indiscriminately for all parameters, a question is arisen: 
\begin{center}
%  \emph{Does SAM help the flat subspace or not?} \\
 \emph{Do we need to calculate perturbations for all parameters?}
\end{center}

Above all, we notice that in most deep neural networks, only about 5\% of parameters are sharp and rise steeply during optimization~\cite{LBtrain}. Then we explore the effect of SAM in different dimensions to answer the above question and find out \emph{(i) little difference between SGD and SAM gradients in most dimensions (see figure~\ref{fig:grad-diff}); (ii) more flatter without SAM in some dimensions (see figure~\ref{fig:vis-landscape} and figure~\ref{fig:hessian}).} 

Inspired by the above discoveries, we propose a novel scheme to improve the efficiency of SAM via sparse perturbation, termed Sparse SAM (SSAM). SSAM, which plays the role of regularization, has better generalization, and its sparse operation also ensures the efficiency of optimization.
Specifically, the perturbation in SSAM is multiplied by a binary sparse mask to determine which parameters should be perturbed. 
Different patterns of masks is discussed and experimented including commonly used unstructured mask, truly accelerated structured mask and generally popular $N$:$M$ structured mask. To achieve true acceleration in our Sparse SAM, we review the backward propagation and classify the position of the mask into two types,~\emph{i.e.}, applied directly on the model parameters called explicit mask, and applied on the intermediate gradient called implicit mask.
To obtain the sparse mask, we provide two implementations. The first solution is to use Fisher information~\cite{fisher-information} of the parameters to formulate the binary mask, dubbed SSAM-F. The other one is to employ dynamic sparse training to jointly optimize model parameters and the sparse mask, dubbed SSAM-D. The first solution is relatively more stable but a bit time-consuming, while the latter is more efficient. 

In addition to these solutions, we provide the theoretical analysis including convergence analysis and generalization analysis. The convergence analysis of SAM and SSAM in non-convex stochastic setting proves that our SSAM can converge at the same rate as SAM,~\emph{i.e.}, $O(\log T/\sqrt{T})$. The generalization analysis shows the generalization error upper bound of our SSAM based on PAC-Bayesian generalization bound~\cite{pac_bayesian}. At last, we evaluate the performance and effectiveness of SSAM on CIFAR10~\cite{cifar10/100}, CIFAR100~\cite{cifar10/100} and ImageNet~\cite{imagenet} with various models. The experiments confirm that SSAM contributes to a flatter landscape than SAM, and its performance is on par with or even better than SAM with only about 50\% perturbation. These results coincide with our motivations and findings.

To sum up, the contribution of this paper is three-fold:
\begin{itemize}
    \item We rethink the role of perturbation in SAM and find that the indiscriminate perturbations are suboptimal and computationally inefficient.
    \item We propose a sparsified perturbation approach called Sparse SAM (SSAM) with two variants,~\emph{i.e.}, Fisher SSAM (SSAM-F) and Dynamic SSAM (SSAM-D), both of which enjoy better efficiency and effectiveness than SAM. We discuss mask with different patterns,~\emph{e.g.}, unstructured, structured, $N$:$M$ structured mask, and positions,~\emph{e.g.}, explicit mask and implicit mask.  
    \item We theoretically prove that SSAM can converge at the same rate as SAM,~\emph{i.e.}, $O(\log T/\sqrt{T})$. We analyze the generalization error upper bound of SSAM. 
    \item We evaluate SSAM with various models on CIFAR and ImageNet, showing WideResNet with SSAM of a high sparsity outperforms SAM on CIFAR; SSAM can achieve competitive performance with a high sparsity; SSAM has a comparable convergence rate to SAM.
\end{itemize}

This manuscript is an extension of our conference version~\cite{ssam}. Our significant improvement includes the investigation of different implementations of sparse perturbation, polishing of our convergence proof, accession of generalization error bound, more comprehensive experiments on noisy dataset, and more ablations.

%% file: texfile/related_work.tex
\section{Related Work}
\label[section]{sec:related-work}

In this section, we briefly review the studies on sharpness-aware minimum optimization (SAM), Fisher information in deep learning, and dynamic sparse training.

\textbf{SAM and flat minima.}
% \label[section]{sec:related-work:sam}
Hochreiter~\emph{et al.}~\cite{flat-minima} first reveal that there is a strong correlation between the generalization of a model and the flat minima. After that, there is a growing amount of research based on this finding. Keskar~\emph{et al.}~\cite{LBtrain} conduct experiments with a larger batch size, and in consequence observe the degradation of model generalization capability. They~\cite{LBtrain} also confirm the essence of this phenomenon, which is that the model tends to converge to the sharp minima. Keskar~\emph{et al.}~\cite{LBtrain} and Dinh~\emph{et al.}~\cite{sharp-can-generalize} state that the sharpness can be evaluated by the eigenvalues of the Hessian. However, they fail to find the flat minima due to the notorious computational cost of Hessian. 

Inspired by this, Foret~\emph{et al.}~\cite{sam} introduce a sharpness-aware optimization (SAM) to find a flat minimum for improving generalization capability, which is achieved by solving a mini-max problem. Zhang~\emph{et al.}~\cite{penalize-sam} make a point that SAM~\cite{sam} is equivalent to adding the regularization of the gradient norm by approximating Hessian matrix. Kwon~\emph{et al.}~\cite{asam} propose a scale-invariant SAM scheme with adaptive radius to improve training stability. Zhang~\emph{et al.}~\cite{surrogate} redefine the landscape sharpness from an intuitive and theoretical perspective based on SAM. To reduce the computational cost in SAM, Du~\emph{et al.}~\cite{esam} proposed Efficient SAM (ESAM) to randomly calculate perturbation. However, ESAM randomly select the samples every steps, which may lead to optimization bias. Instead of the perturbations for all parameters,~\emph{i.e.}, SAM, we compute a sparse perturbation,~\emph{i.e.}, SSAM, which learns important but sparse dimensions for perturbation.

\textbf{Fisher information (FI).} Fisher information
\cite{fisher-information} was proposed to measure the information that an observable random variable carries about an unknown parameter of a distribution.
% , defined as:
% \begin{equation}
% \label[equation]{equ:fisher}
%     F_{\w} = \E_{x\sim p(x)}\left[ \E_{y\sim p_{\w}(y|x)}\nabla_{\w}\log p_{\w}(y|x) \nabla_{\w}\log p_{\w}(y|x)^T\right],
% \end{equation}
% where $p_{\w}(y|x)$ is the output distribution predicted by the model with its weight $\w$ given input $x$.
In machine learning, Fisher information is widely used to measure the importance of the model parameters~\cite{pnas-forgetting-in-neural-networks} and decide which parameter to be pruned~\cite{fisher-mask, fisher-pruning}. 
For proposed SSAM-F, Fisher information is used to determine whether a weight should be perturbed for flat minima.  

\textbf{Dynamic sparse training.}
% \label[section]{sec:related-work:dst}
Finding the sparse network via pruning unimportant weights is a popular solution in network compression, which can be traced back to decades~\cite{lecun-1th-prune}. The widely used training scheme, \emph{i.e.}, pretraining-pruning-fine-tuning, is presented by Han~\emph{et.al.}~\cite{prune-pipeline}. Limited by the requirement for the pre-trained model, some recent research~\cite{rig,dst-recent-deep-rewiring,sparse-cosine,dst-recent-topkast,dst-recent-reparameterization,liu2021sparse,liu2022unreasonable} attempts to discover a sparse network directly from the training process. Dynamic Sparse Training (DST) finds the sparse structure by dynamic parameter reallocation. The criterion of pruning could be weight magnitude~\cite{LTH}, gradient~\cite{rig} and Hessian~\cite{lecun-1th-prune, woodfisher}, \emph{etc}.
We claim that different from the existing DST methods that prune neurons, our target is to obtain a binary mask for sparse perturbation. 

\textbf{Efficient training backpropagation.} Efficient training is proposed to achieve true acceleration to reduce training budget, since unstructured sparse is not supported by hardware. Besides structured sparse network~\cite{molchanov2019importance, filter_prune, huang2018learning, bi-nm}, the most popular pattern achieving true efficient train is $N$:$M$ structure~\cite{nvidia-ampere-tensor-core,n:M-structured-sparse,nvidia-whitepaper, n:M-transposable-mask, channel-n-m, sr-ste}, which requires $M$ of every $N$ consecutive elements are 0. Hubara~\emph{et al.}~\cite{n:M-transposable-mask} provides an analysis on the $N$:$M$ case at forward propagation and backward propagation. Pool~\emph{et al.}~\cite{channel-n-m} apply the $N$:$M$ pattern on CNN via permuting the input channels. Zhang~\emph{et al.}~\cite{bi-nm} proposed two $N$:$M$ masks for forward propagation and backward propagation, respectively.

%% file: texfile/rethinking_sam.tex
\section{Rethinking the Perturbation in SAM}
\label[section]{sec:rethinking}
In this section, we first review how SAM converges at the flat minimum of a loss landscape. Then, we rethink the role of perturbation in SAM.
 
\subsection{Preliminary}
\label[section]{sec:rethinking:preliminary}

In this paper, we consider the weights of a deep neural network as $\w=(w_1,w_2,...,w_d) \subseteq \mathcal{W}\in \mathbb{R}^d$ and denote a binary mask as $\m \in \{0,1\}^d$, which satisfies $\boldsymbol{1}^T\m=(1-s)\cdot d$ to restrict the computational cost. Given a training dataset as $\mathcal{S} \triangleq \{(\xi,\yi)\}_{i=1}^n$ \emph{i.i.d.} drawn from the distribution $\mathcal{D}$, the per-data-point loss function is defined by $f(\w, \xi, \yi)$. For the classification task in this paper, we use cross-entropy as loss function. The population loss is defined by $f_{\mathcal{D}}=\E_{(\xi,\yi)\sim\mathcal{D}}f(\w,\xi,\yi)$, while the empirical training loss function is $f_{\mathcal{S}}\triangleq \frac{1}{n}\sum_{i=1}^nf(\w,\xi,\yi)$.

Sharpness-aware minimization (SAM)~\cite{sam} aims to simultaneously minimize the loss value and smooth the loss landscape, which is formulated as a min-max problem:
\begin{equation}
\label[equation]{equ:sam}
    \min_{\w} \max_{||\boldsymbol{\epsilon}||_2\leq\rho} f_{\mathcal{S}}(\w + \boldsymbol{\epsilon}).
\end{equation}
SAM first obtains the perturbation $\boldsymbol{\epsilon}$ in a neighborhood ball area with a radius denoted as $\rho$. The optimization tries to minimize the loss of the perturbed weight $\w + \boldsymbol{\epsilon}$. Intuitively, the goal of SAM is that small perturbations to the weight will not significantly rise the empirical loss, which indicates that SAM tends to converge to a flat minimum. To solve the mini-max optimization, SAM approximately calculates the perturbations $\boldsymbol{\epsilon}$ using Taylor expansion around $\w$:
\begin{align}
\label[equation]{equ:eps}
\boldsymbol{\epsilon} =&\mathop{\arg\max}_{||\boldsymbol{\epsilon}||_2\leq\rho} f_{\mathcal{S}}(\w + \boldsymbol{\epsilon})
\notag 
\approx \mathop{\arg\max}_{||\boldsymbol{\epsilon}||_2\leq\rho} f_{\mathcal{S}}(\w) + \boldsymbol{\epsilon} \cdot \nabla_{\w}f(\w) 
\notag \\ 
=& \rho \cdot {\nabla_{\w}f(\w)}{\big /}{||\nabla_{\w}f(\w)||_2}.
\end{align}
In this way, the objective function can be rewritten as $\min_{\w}f_{\mathcal{S}}\big(\w+\rho {\nabla_{\w}f(\w)}{\big/}{||\nabla_{\w}f(\w)||_2}\big)$, which could be implemented by a two-step gradient descent framework in Pytorch or TensorFlow:
\begin{itemize}
    \item In the first step, the gradient at $\w$ is used to calculate the perturbation $\boldsymbol{\epsilon}$ by Eq. \ref{equ:eps}. Then the weight of model will be added to $\w + \boldsymbol{\epsilon}$.
    \item In the second step, the gradient at $\w + \boldsymbol{\epsilon}$ is used to solve $\min_{\w} f_{\mathcal{S}}(w+\boldsymbol{\epsilon})$, \emph{i.e.}, update the weight $\w$ by this gradient.
\end{itemize}

\subsection{Rethinking the Perturbation Step of SAM}
\label[section]{sec:rethinking:rethinkSAM}

\textbf{How does SAM work in flat subspace?} SAM perturbs all parameters indiscriminately, but the fact is that merely about 5\% parameter space is sharp while the rest is flat~\cite{LBtrain}. We are curious whether perturbing the parameters in those already flat dimensions would lead to the instability of the optimization and impair the improvement of generalization. To answer this question, we quantitatively and qualitatively analyze the loss landscapes with different training schemes in~\cref{sec:experiments}, as shown in ~\cref{fig:vis-landscape} and~\cref{fig:hessian}. The results confirm our conjecture that optimizing some dimensions without perturbation can help the model generalize better. 
% \textbf{More flatter without SAM in some dimensions.} 
% To explore the smoothing ability of SAM on different dimensions, we use SGD, SAM and their Mix\footnote{The mix indicates the sparse perturbations, which computes perturbations in a few dimensions,~\emph{i.e.}, optimized by SAM, while the most of other dimensions do not compute perturbations,~\emph{i.e.}, optimized by SGD.} to optimize ResNet18 on CIFAR10 dataset with cross entropy loss. 
% After that, we visualize the training loss landscape in~\cref{fig:vis-landscape-mix}, and find that the training loss landscape of using the mix of SGD and SAM is more flatter showing a majority of low-loss areas,~\emph{i.e.}, blue areas. In addition, all of them are sampled in the same range to get the loss landscapes, and the training loss increased of the mix is minimal,~\emph{i.e.}, 3.5, within the same sampling range, which also confirms that some dimensions could be flatter without SAM. It turns out that using SAM in all dimension is not optimal and may damage the flatness the model.
% \begin{figure}[ht]
%     \centering
%     \includegraphics[width=1\linewidth]{figs/vis-landscape-mix.pdf}
%     \vspace{-0.4cm}
%      \caption{Training loss landscapes of ResNet18 on CIFAR10 with SGD, SAM and the mix of two.}
%      \label{fig:vis-landscape-mix}
%     \vspace{-0.4cm}
% \end{figure}

\textbf{What is the difference between the gradients of SGD and SAM?}
We investigate various neural networks optimized with SAM and SGD on CIFAR10/100 and ImageNet, whose statistics are given in~\cref{fig:grad-diff}. We use the relative difference ratio $r$, defined as $r = \log \left|{(g_{SAM}- g_{SGD})}/{g_{SGD}}\right|$, to measure the difference between the gradients of SAM and SGD. As showin in~\cref{fig:grad-diff}, the parameters with $r$ less than 0 account for the vast majority of all parameters, indicating that most SAM gradients are not significantly different from SGD gradients. These results show that most parameters of the model require no perturbations for achieving the flat minima, which well confirms our the motivation. 
\begin{figure}[h]
    \centering
    \vspace{-0.2cm}
    \includegraphics[width=0.98\linewidth]{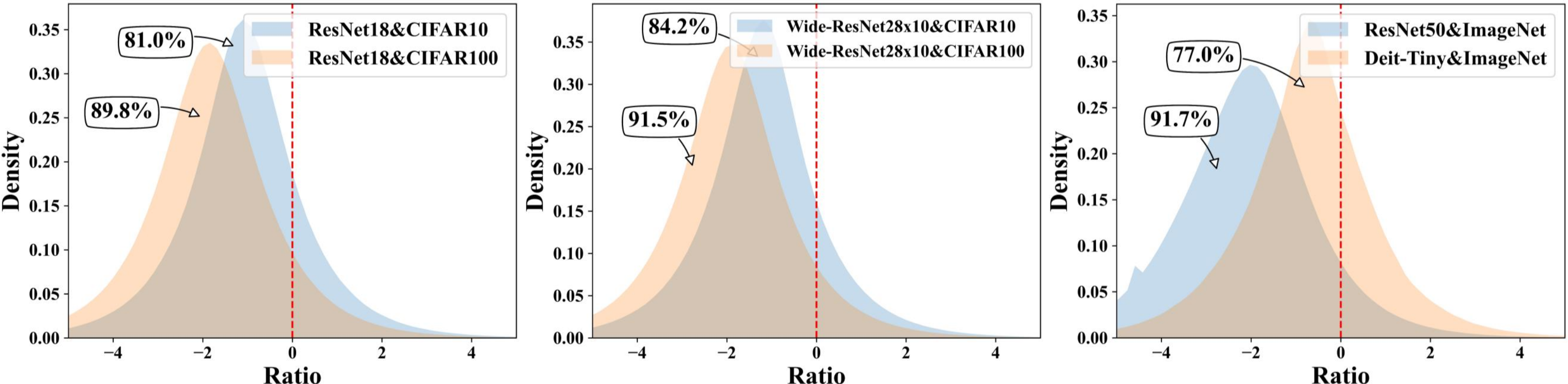}
    \caption{The distribution of relative difference ratio $r$ among various models and datasets. There is little difference between SAM and SGD gradients for most parameters,~\emph{i.e.}, the ratio $r$ is less than $0$.}
    \label[figure]{fig:grad-diff}
    \vspace{-0.4cm}
\end{figure}

% Despite the improvement of generalization, the indiscriminate perturbation not only make less effective in some dimensions, but also increase the unnecessary computation in most dimensions. 
Inspired by the above observation and the promising hardware acceleration for sparse operation modern GPUs, we further propose Sparse SAM, a novel sparse perturbation approach, as an implicit regularization to improve the efficiency and effectiveness of SAM.

%% file: texfile/methodology.tex
\section{Methodology}
\label[section]{sec:method}

\begin{figure*}
{
\label{illustration_of_pattern_and_flow}
\centering
\subfloat[Illustration of different patterns of mask applied on perturbation for a sparse computation. (1) The original computation of SAM with requirements of a complete gradient. (2) Unstructured pattern does not require the position of the elements in the mask as long as it satisfies the sparsity ratio, while the implementation of true acceleration on hardware is difficult to achieve. (3) Structured pattern requires the elements of the mask should be in a row/filter and it's easy to be realized but harmful to performance. (4) 2:4 structure pattern requires that there are only two weights are nonzero in each group of four weights, which is being gradually developed and supported by hardware.]{
    \centering
    \includegraphics[scale=0.5, height=5cm]{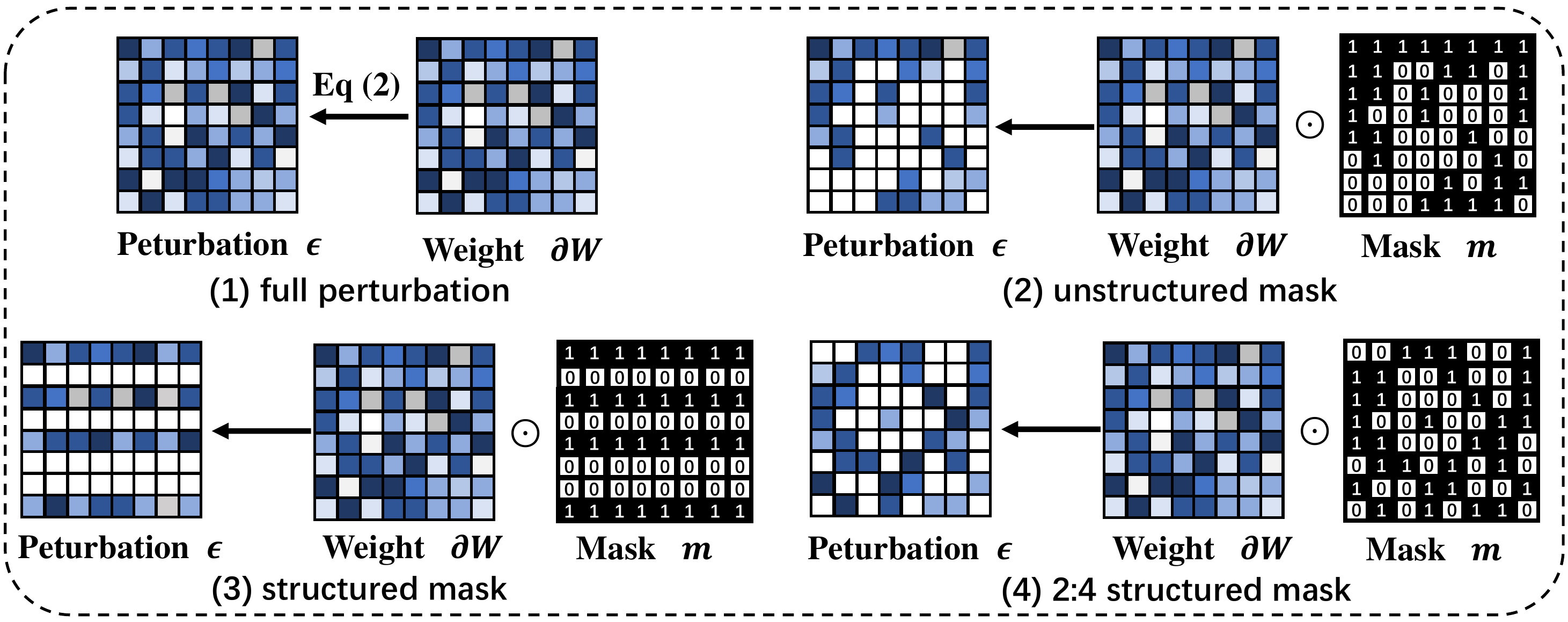}
    \label{mask_pattern}
}
\hspace{+2mm}
\subfloat[Implement sparse perturbation Explicitly and Implicitly. (1) Explicit implementation,~\emph{i.e.}, apply the mask on the gradient of weight, access sparse perturbation directly. (2) Implicit implementation,~\emph{i.e.}, apply the mask on the gradient of intermediate variable, access sparse perturbation indirectly.]{
    \centering
    \includegraphics[scale=0.3, height=5cm]{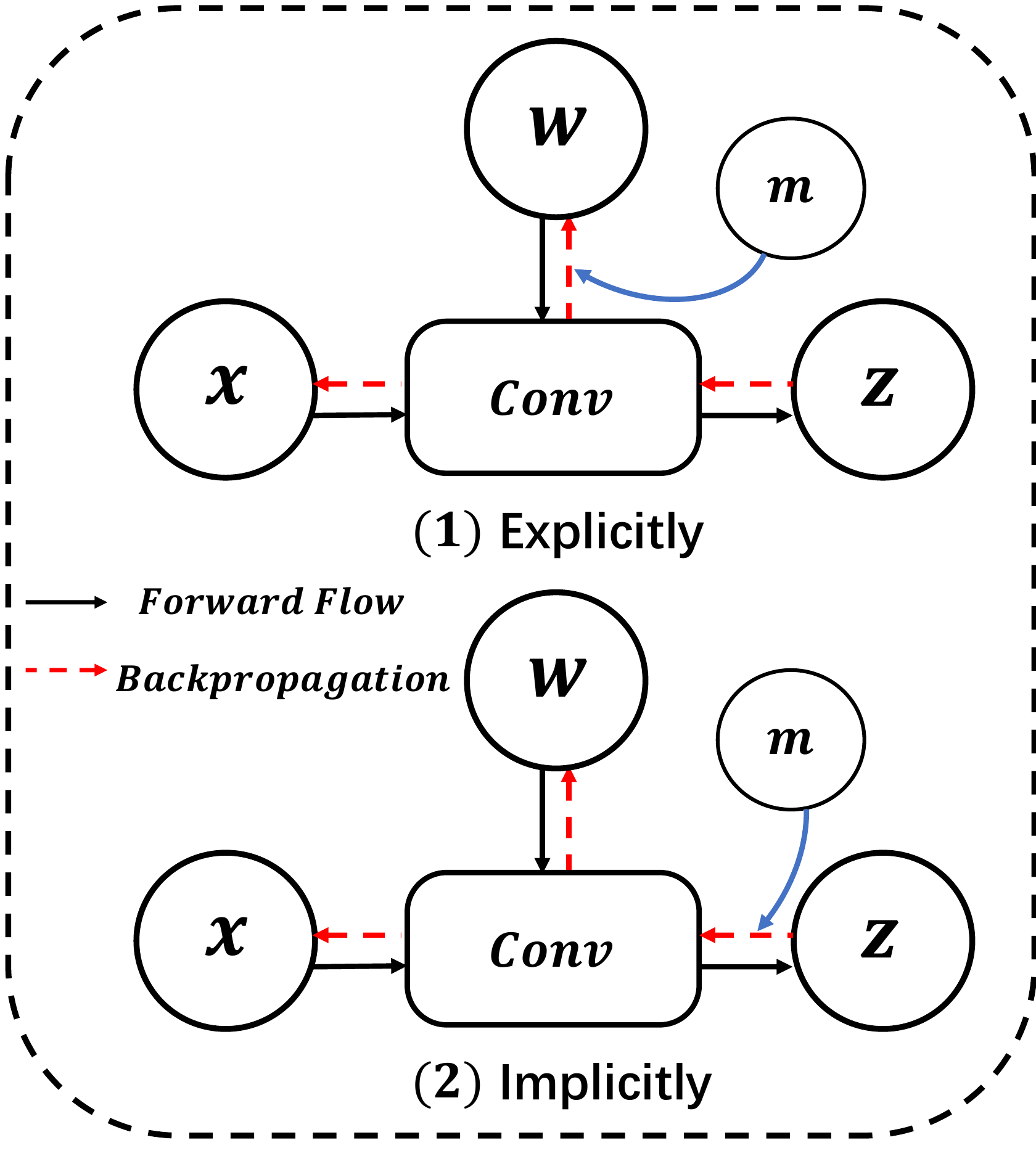}
    \label{train_flow}
}
\caption{Illustration of different patterns of masks, and explicit and implicit implementation of sparse perturbation.}
}
\end{figure*}

In this section, we first define the proposed Sparse SAM (SSAM), which strengths SAM via sparse perturbation. Afterwards, we introduce the instantiations of the sparse mask used in SSAM via Fisher information and dynamic sparse training, dubbed SSAM-F and SSAM-D, respectively. 

\subsection{Sparse SAM}
\label[section]{sec:method:sparseSAM}
Motivated by the finding discussed in the introduction, Sparse SAM (SSAM) employs a sparse binary mask to decide which parameters should be perturbed, thereby improving the efficiency of sharpness-aware minimization.  
Specifically, the perturbation $\boldsymbol{\epsilon}$ will be multiplied by a sparse binary mask $\m$, and the objective function is then rewritten as $\min_{\w}f_{\mathcal{S}}\left(\w+\rho\cdot\frac{\nabla_{\w} f(\w)}{||\nabla_{\w}f(\w)||_2}\odot \m\right)$.  To stable the optimization, the sparse binary mask $\m$ is updated at regular intervals during training. We provide two solutions to obtain the sparse mask $\m$, namely Fisher information based Sparse SAM (SSAM-F) and dynamic sparse training based Sparse SAM (SSAM-D). The overall algorithms of SSAM and sparse mask generations are shown in Algorithm~\ref{alg:sparse-sam-pipeline} and Algorithm~\ref{alg:sparse-sam-generate-mask}, respectively.

    \begin{algorithm}[ht]
    \small
        \caption{Sparse SAM (SSAM)}
        \label{alg:sparse-sam-pipeline}
    \begin{algorithmic}[1]
    \REQUIRE sparse ratio $s$, dense model $\w$, binary mask $\m$, update interval $T_m$, number of samples $N_F$, learning rate $\eta$, training set $\mathcal{S}$.
    \STATE Initialize $\w$ and $\m$ randomly.
    \FOR{epoch $t=1,2 \ldots T$}
        \FOR{each training iteration}
        \STATE{Sample a batch from $\mathcal{S}$: $\mathcal{B}$}
        % \STATE{Compute gradient $\nabla f(\w)$}
        \STATE{Compute perturbation $\boldsymbol{\epsilon}$ by Eq.~\eqref{equ:eps}}
        \STATE{$\boldsymbol{\epsilon}\gets\boldsymbol{\epsilon} \odot \m$}
        \ENDFOR
    % \STATE{Compute SAM gradient $\nabla f(\w+\boldsymbol{\epsilon})$}
    \IF{$t \mod T_m = 0$}
    \STATE Generate mask $m$ via {\color{blue}{Option I}} or {\color{blue}{II}}
    \ENDIF
    \STATE{$\w \gets \w - \eta \nabla f(\w+\boldsymbol{\epsilon})$}
    \ENDFOR
    \RETURN{Final weight of model $\w$}
    \end{algorithmic}
    \end{algorithm}
    \begin{algorithm}[ht]
    \small
         \caption{Sparse Mask Generation}
        \label{alg:sparse-sam-generate-mask}
    \begin{algorithmic}[1]
    \STATE  {\color{blue}{Option I:(Fisher Information Mask)}}
        \STATE{Sample $N_F$ data from $\mathcal{S}$: $\mathcal{B}_F$}
        \STATE{Compute Fisher $\hat{\mathscr{F}}_{\w} $ by Eq.~\eqref{equ:fisher-empirical}}
        \STATE{$\boldsymbol{m_1}\!=\!\{m_i|m_i\!=\!1\}\gets {\rm ArgTopK}(\hat{\mathscr{F}}_{\w}, (1-s)\cdot|\w|)$}
        \STATE{$\boldsymbol{m_0}\!=\!\{m_i|m_i\!=\!0\}\gets \boldsymbol{m} - \boldsymbol{m_1}$}
        \STATE{$\m \gets \boldsymbol{m_0} \cup \boldsymbol{m_1}$}
    \STATE  {\color{blue}{Option II:(Dynamic Sparse Mask)}}    
        \STATE{$N_{drop}=f_{decay}(t;\alpha)\cdot (1 - s) \cdot |\w|$}
        \STATE{$N_{growth}=N_{drop}$}
        \STATE{$\boldsymbol{m_1}
        \!\gets\!\boldsymbol{m_1}\!-\! {\rm ArgTopK}_{m_i\!\in\!\boldsymbol{m_1}}(-|\nabla f(\w)|, N_{drop})$}
        \STATE{$\boldsymbol{m_1}\!\gets\!\boldsymbol{m_1} +{\rm Random}_{\m_i\!\notin\!\boldsymbol{m_1}}(N_{growth})$}
        \STATE{$\boldsymbol{m_0}\!=\!\{m_i|m_i\!=\!0\}\gets \m - \boldsymbol{m_1}$}
        \STATE{$\m \gets \boldsymbol{m_0} \cup \boldsymbol{m_1}$}
    \RETURN Sparse mask $\m$
    \end{algorithmic}
    \end{algorithm}

% According to the previous work~\cite{2:4-truly-sparse} and Ampere architecture equipped with sparse tensor cores~\cite{nvidia-whitepaper, nvidia-asp, nvidia-ampere-tensor-core}, currently there exists technical support for matrix multiplication with 50\% fine-grained sparsity~\cite{2:4-truly-sparse}\footnote{For instance, 2:4 sparsity for A100 GPU.}. Therefore, SSAM of 50\% sparse perturbation has great potential to achieve true training acceleration via sparse back-propagation.

\subsection{SSAM-F: Fisher information based Sparse SAM}
\label[section]{sec:method:sparseSAM:fisherSAM}

Inspired by the connection between Fisher information and Hessian~\cite{fisher-information}, which can directly measure the flatness of the loss landscape, we apply Fisher information to achieve sparse perturbation, denoted as SSAM-F. The Fisher information is defined by
\begin{equation}
\small
\label[Equation]{equ:fisher-information-matrix}
    \mathscr{F}_{\w} = \E_{x\sim p(x)}\left[ \E_{y\sim p(y|x)}\nabla_{\w}\log p(y|x) \nabla_{\w}\log p(y|x)^T\right]
\end{equation} 
where $p(y|x)$ is the abbreviation of $p_{\w}(y|x)$, which is the output distribution predicted by the model. In over-parameterized networks, the computation of the Fisher information matrix is also intractable, \emph{i.e.}, $|\w|\times|\w|$. Following~\cite{fisher-mask, KFAC, KFAC-opt, KFAC-cnn}, we approximate $\mathscr{F}_{\w}$ as a diagonal matrix, which is included in a vector in $\mathbb{R}^{|\w|}$. Note that there are two expectations in Eq.~\eqref{equ:fisher-information-matrix}. The first one is that the original data distribution $p(x)$, which is often not available. Therefore, we approximate it by sampling $N_F$ data $\x_1,\x_2,\ldots, \x_{N_F}$ from $x\sim p(x)$:
\begin{equation}
\label[equation]{equ:fisher-approximation-x}
    \mathscr{F}_{\w}\approx \frac{1}{N_F}\sum_{i=1}^{N_F}\E_{y\sim p_{\w}(y|\x_i)}(\nabla_{\w}\log p_{\w}(y|\x_i))^2.
\end{equation}
For the second expectation over $p_{\w}(y|\x)$, it is not necessary to compute its explicit expectation, since the ground-truth $y_i$ for each training  sample $x_i$ is available in supervised learning. Therefore, we rewrite the Eq.~\eqref{equ:fisher-approximation-x} as  "empirical Fisher":
\begin{equation}
\label[equation]{equ:fisher-empirical}
\hat{\mathscr{F}}_{\w}=\frac{1}{N_F}\sum_{i=1}^{N_F}(\nabla_{\w}\log p_{\w}(y_i|\x_i))^2.
\end{equation}
We emphasize that the empirical Fisher is a $|\w|$-dimension vector, which is the same size as the mask $\m$. To obtain the mask $\m$, we calculate the empirical Fisher by Eq.~\eqref{equ:fisher-empirical} over $N_F$ training data randomly sampled from training set $\mathcal{S}$. Then we sort the elements of empirical Fisher in descending, and the parameters corresponding to the the top $k$ Fisher values will be perturbed:
\begin{equation}
\boldsymbol{m_1}=\{m_i|m_i=1,m_i\in\m, i\in{\rm ArgTopK}(\hat{\mathscr{F}}_{\w}, k)\},
\end{equation}
where ${\rm ArgTopK}(\boldsymbol{v}, N)$ returns the indices of top $N$ largest elements among $\boldsymbol{v}$, $\boldsymbol{m_1}$ is the set of values that are 1 in $\m$. $k$ is the number of perturbed parameters, which is equal to $(1-s)\cdot |\w|$ for sparsity $s$. After setting the rest values of the mask to 0,~\emph{i.e.}, $\boldsymbol{m_0}=\{m_i| m_i=0,m_i\notin \boldsymbol{m_1}, m_i\in \m \}$, we get the final mask $\m=\boldsymbol{m_0}\cup\boldsymbol{m_1}$. The algorithm of SSAM-F is shown in \cref{alg:sparse-sam-generate-mask}.

\subsection{SSAM-D: Dynamic sparse mask based sparse SAM}
\label[section]{sec:method:sparseSAM:dstSAM}
Considering the computation of empirical Fisher is still relatively high, we also resort to dynamic sparse training for efficient binary mask generation. The mask generation includes a perturbation dropping phase and a perturbation growth phase. At the perturbation dropping phase, some perturbed parameters will be dropped, which means that they require no perturbations. 
We believe that the advantage brought by SAM for the parameters located in relatively flat positions is no longer effective, as SAM is intended to smooth landscape. Therefore, the parameters with small gradients values would be dropped. Formally, the perturbation dropping phase follows
\begin{equation}
    \boldsymbol{m_1} \gets \boldsymbol{m_1} - {\rm ArgTopK}_{\w \in \boldsymbol{m_1}}(-|\nabla f(\w)|, N_{drop}),
\end{equation}
where $\boldsymbol{m_1}$ represents the parameters set whose mask is 1, $N_{drop}$ is the number of perturbations to be dropped. At the perturbation growth phase, for the purpose of exploring the perturbation combinations as many as possible, several unperturbed dimensions grow, which means these dimensions need to compute perturbations. Formally, the perturbation growth phase follows
\begin{equation}
    \boldsymbol{m_1} \gets \boldsymbol{m_1} + {\rm Random}_{\w \notin \boldsymbol{m_1}}(N_{growth}),
\end{equation}
where ${\rm Random}_{\mathcal{S}}(N)$ randomly returns $N$ elements in $\mathcal{S}$, and $N_{growth}$ is the number of perturbation growths. To keep the sparsity constant during training, the number of growths is equal to the number of dropping,~\emph{i.e.}, $N_{growth}=N_{drop}$. Afterwards, we set the rest values of the mask to $0$,~\emph{i.e.}, $\boldsymbol{m_0}\!=\!\{m_i\!=\!0|m_i\!\notin\!\boldsymbol{m_1}\}$, and get the final mask $\m\!=\!\boldsymbol{m_0}\cup\boldsymbol{m_1}$. The drop ratio $\alpha$ represents the proportion of dropped perturbations in the total perturbations $s\cdot|\w|$,~\emph{i.e.}, $\alpha=N_{drop}/\left((1-s)\cdot|\w|\right)$. In particular, a larger drop rate means that more combinations of binary mask can be explored during optimization, which, however, may slightly interfere the optimization process. Following \cite{rig, sparse-cosine}, we apply a cosine decay scheduler to alleviate this problem:
\begin{equation}
    f_{decay}(t;\alpha) = \frac{\alpha}{2}\left( 1+\cos \left( {t\pi}{/}{T} \right)  \right),
\end{equation}
where $T$ denotes number of training epochs. The algorithm of SSAM-D is depicted in ~\cref{alg:sparse-sam-generate-mask}.

\subsection{Masking with different patterns and explicit or implicit implementation of sparse perturbation}
In this subsection, we first discuss the pattern of mask,~\emph{e.g.}, \emph{unstructured} mask, \emph{structured} mask and \emph{N:M structured} mask. Then we discuss the position to be masked.

\textbf{Different patterns.} Similar to model pruning, the mask has different constraints to achieve acceleration. As shown in figure~\ref{mask_pattern}, we illustrate the diagram of different patterns of masks including \emph{unstructured} mask, \emph{structured} mask, and \emph{2:4 structured} mask. An unstructured mask has no requirements for the elements, which is more flexible and common-case in sparse training. Nevertheless, the unstructured pattern is not hardware-friendly and it can not be calculated rapidly. A structured mask normally appears in the convolution operator, where the unit to be masked is the filter. Structured patterns request a whole row or column or filter be masked together, which could achieve a true acceleration. $N$: $M$ structured pattern requires M out of every N consecutive elements in the matrix are 0,~\emph{e.g.}, 2:4 pattern shown in~\figref{mask_pattern}.

\textbf{Explicit and implicit implementation.}
We first briefly review backpropagation. Take the linear layer $\boldsymbol{A}_{n\times m}$ in a neural network as an example,~\emph{i.e.},
\begin{align}
\label{forward}
    \boldsymbol{y}_{n\times 1} = \boldsymbol{A}_{n\times m} \boldsymbol{x}_{m \times 1}.
\end{align}
where $\boldsymbol{x}$ is the output of the last layer and $\boldsymbol{y}$ would be the input of the next layer. At the backpropagation step, it calculates the derivative of loss $f$ with respect to each of these terms
\begin{align}
    & \frac{\partial f}{\partial \boldsymbol{A}}_{n\times m} =   \frac{\partial f}{\partial \boldsymbol{y}}_{n\times 1} \boldsymbol{x}^T_{1\times m}
    \\
    & \frac{\partial f}{\partial \boldsymbol{x}}_{m\times 1} =   \boldsymbol{A}^T_{m\times n} \frac{\partial f}{\partial \boldsymbol{y}}_{n\times 1}  \label{backward_intermediate}
\end{align}
The first term $\frac{\partial f}{\partial \boldsymbol{A}}$ is the well-known gradient of parameters,~\emph{i.e.}, $\nabla f(\w)$, which directly participates in the step of updating the model parameters. Intuitively, the sparse mask on gradient $\frac{\partial f}{\partial \boldsymbol{A}}$ could directly access sparse perturbation shown in figure~\ref{train_flow}. We can formalize it explicitly
\begin{align}
    \boldsymbol{\epsilon} \propto \frac{\partial f}{\partial \boldsymbol{A}} \odot  \m = \left(  \frac{\partial f}{\partial \boldsymbol{y}}\boldsymbol{x}^T\right) \odot  \m
\end{align}
which is mainly used in our method and subsequent experiments. 

The second one $\frac{\partial f}{\partial \boldsymbol{x}}$ is the intermediate gradient for previous layer, which dominates most of the time in back propagation. Therefore, to achieve actually accelerated SSAM, we try to apply the mask on the intermediate gradient in backward propagation, shown in figure~\ref{train_flow}, which could be considered as implicit sparse perturbation. We also conduct the corresponding experiments in following section.

%%%%%%%%%%%%%%%%%%%%%%%%%%%%%%%%%%%%%%%%%%%%%%%%%%%%%%%%%%%%%%%%%%%%%%%%%%%%%
%%%%%%%%%%%%%%%%%%%%%%%%%%%%%%%%%%%%%%%%%%%%%%%%%%%%%%%%%%%%%%%%%%%%%%%%%%%%%
%%%%%%%%%%%%%%%%%%%%%%%%%%%%%%%%%%%%%%%%%%%%%%%%%%%%%%%%%%%%%%%%%%%%%%%%%%%%%
%%%%%%%%%%%%%%%%%%%%%%%%%%%%%%%%%%%%%%%%%%%%%%%%%%%%%%%%%%%%%%%%%%%%%%%%%%%%%
%%%%%%%%%%%%%%%%%%%%%%%%%%%%%%%%%%%%%%%%%%%%%%%%%%%%%%%%%%%%%%%%%%%%%%%%%%%%%
\begin{table*}[ht]
\centering
\caption{Test accuracy of ResNet on CIFAR with unstructured Sparse SAM.}
\label[table]{resnet_cifar}
\begin{tabular}{ccccccc}
\toprule
Model & Optimizer & Sparsity & \multicolumn{2}{c}{CIFAR10} & \multicolumn{2}{c}{CIFAR100} \\ \hline
\multirow{9}{*}{ResNet18} & SGD & / & \multicolumn{2}{c}{96.07\%} & \multicolumn{2}{c}{77.80\%} \\
 & SAM & 0\% & \multicolumn{2}{c}{96.83\%} & \multicolumn{2}{c}{81.03\%}  \\ \cline{2-7} 
 &  &  & SSAM-F & SSAM-D & SSAM-F & SSAM-D  \\ \cline{4-7} 
 & \multirow{6}{*}{SSAM(Ours)} & 50\% & \textbf{96.81\% \textcolor{blue}{(-0.02)}} & \textbf{96.87\% \textcolor{blue}{(+0.04)}} & \textbf{81.24\% \textcolor{blue}{(+0.21)}} & \textbf{80.59\% \textcolor{blue}{(-0.44)}}  \\
 &  & 80\% & 96.64\% \textcolor{blue}{(-0.19)} & 96.76\% \textcolor{blue}{(-0.07)} & 80.47\%  \textcolor{blue}{(-0.56)} & 80.43\% \textcolor{blue}{(-0.60)} \\
 &  & 90\% & 96.75\% \textcolor{blue}{(-0.08)} & 96.67\% \textcolor{blue}{(-0.16)} & 80.02\%  \textcolor{blue}{(-1.01)} & 80.39\% \textcolor{blue}{(-0.64)} \\
 &  & 95\% & 96.66\% \textcolor{blue}{(-0.17)} & 96.56\% \textcolor{blue}{(-0.27)} & 80.50\%  \textcolor{blue}{(-0.53)} & 79.79\% \textcolor{blue}{(-1.24)} \\
 &  & 98\% & 96.55\% \textcolor{blue}{(-0.28)} & 96.61\% \textcolor{blue}{(-0.22)} & 80.09\%  \textcolor{blue}{(-0.94)} & 79.79\% \textcolor{blue}{(-1.24)}  \\
 &  & 99\% & 96.52\% \textcolor{blue}{(-0.31)} & 96.59\% \textcolor{blue}{(-0.24)} & 80.07\% \textcolor{blue}{(-0.96)} & 79.61\% \textcolor{blue}{(-1.42)} \\ 
 \hline
\multirow{9}{*}{WideResNet28-10} & SGD & / & \multicolumn{2}{c}{97.11\%} & \multicolumn{2}{c}{81.93\%} \\
 & SAM & 0\% & \multicolumn{2}{c}{97.48\%} & \multicolumn{2}{c}{84.20\%} \\ \cline{2-7} 
 &  &  & SSAM-F & SSAM-D & SSAM-F & SSAM-D   \\ \cline{4-7}
 & \multirow{6}{*}{SSAM(Ours)} & 50\% & \textbf{97.71\% \textcolor{blue}{(+0.23)}} & 97.70\% \textcolor{blue}{(+0.22)} & \textbf{85.16\% \textcolor{blue}{(+0.96)}} & \textbf{84.99\% \textcolor{blue}{(+0.79)}} \\
 &  & 80\% & 97.67\% \textcolor{blue}{(+0.19)} & \textbf{97.72\% \textcolor{blue}{(+0.24)}} & 84.57\%  \textcolor{blue}{(+0.37)} & 84.36\% \textcolor{blue}{(+0.16)}  \\
 &  & 90\% & 97.47\% \textcolor{blue}{(-0.01)} & 97.53\% \textcolor{blue}{(+0.05)} & 84.76\%  \textcolor{blue}{(+0.56)} & 84.16\% \textcolor{blue}{(-0.04)} \\
 &  & 95\% & 97.42\% \textcolor{blue}{(-0.06)} & 97.52\% \textcolor{blue}{(+0.04)} & 84.17\%  \textcolor{blue}{(-0.03)} & 83.66\% \textcolor{blue}{(-0.54)}  \\
 &  & 98\% & 97.32\% \textcolor{blue}{(-0.16)} & 97.30\% \textcolor{blue}{(-0.18)} & 83.85\%  \textcolor{blue}{(-0.35)} & 83.30\% \textcolor{blue}{(-0.90)}  \\
 &  & 99\% & 97.59\% \textcolor{blue}{(+0.11)} & 97.27\% \textcolor{blue}{(-0.25)} & 84.00\% \textcolor{blue}{(-0.20)} & 84.16\% \textcolor{blue}{(-0.04)}  \\ 
 \bottomrule
\end{tabular}
\end{table*}
%%%%%%%%%%%%%%%%%%%%%%%%%%%%%%%%%%%%%%%%%%%%%%%%%%%%%%%%%%%%%%%%%%%%%%%%%%%%%
%%%%%%%%%%%%%%%%%%%%%%%%%%%%%%%%%%%%%%%%%%%%%%%%%%%%%%%%%%%%%%%%%%%%%%%%%%%%%
%%%%%%%%%%%%%%%%%%%%%%%%%%%%%%%%%%%%%%%%%%%%%%%%%%%%%%%%%%%%%%%%%%%%%%%%%%%%%
%%%%%%%%%%%%%%%%%%%%%%%%%%%%%%%%%%%%%%%%%%%%%%%%%%%%%%%%%%%%%%%%%%%%%%%%%%%%%
%%%%%%%%%%%%%%%%%%%%%%%%%%%%%%%%%%%%%%%%%%%%%%%%%%%%%%%%%%%%%%%%%%%%%%%%%%%%%

\subsection{Theoretical analysis of Sparse SAM}
\label[section]{sec:method:theoretical}
In the following, we analyze the convergence of SAM and SSAM in non-convex stochastic setting. Before introducing the main theorem, we first describe the following assumptions that are commonly used for characterizing the convergence of nonconvex stochastic optimization \cite{con-adv, surrogate, understand-sam,chen2022towards,bottou2018optimization,ghadimi2013stochastic}.
\begin{assumption}
\label{assume:bounded-gradient}
(Bounded Gradient.) It exists non-negative real number $G \geq 0$ s.t. $||\nabla f(\w)|| \leq G$.
\end{assumption}

\begin{assumption}
\label{assume:bounded-variance}
(Bounded Variance.) It exists non-negative real number $\sigma \geq 0$ s.t. $\E [||g(\w) - \nabla f(\w)||^2] \leq \sigma^2$.
\end{assumption}

\begin{assumption}
\label{assume:l-smoothness}
(L-smoothness.) It exists non-negative real number $L > 0 $ s.t. $||\nabla f(\w) - \nabla f(\boldsymbol{v})|| \leq L||\w - \boldsymbol{v}|| $, $\forall  \w, \boldsymbol{v} \in \mathbb{R}^d$.
\end{assumption}

\begin{theorem}
\label{theorem:sam}
Consider function $f(\w)$ satisfying the Assumptions
\ref{assume:bounded-gradient}-\ref{assume:l-smoothness} optimized by SAM. Let $\eta_t = \frac{\eta_0}{\sqrt{t}}$ and perturbation amplitude $\rho$ decay with square root of $t$, \emph{e.g.}, $\rho_t=\frac{\rho_0}{\sqrt{t}}$. With $\rho_0 \leq \frac{1}{2}G \eta_0$, we have
\begin{equation}
    \frac{1}{T} \sum_{t=0}^T \E ||\nabla f(\w_t)||^2 \leq C_1 \frac{1}{\sqrt{T}} + C_2\frac{\log T}{\sqrt{T}},
\end{equation}
where $C_1=\frac{2}{\eta_0}(f(\w_0)-\E f(\w_T))$ and $C_2=2(L \sigma^2 \eta_0 + LG\rho_0)$. 
\end{theorem}

\begin{theorem}
\label{theorem:ssam}
Consider function $f(\w)$ satisfying the Assumptions
\ref{assume:bounded-gradient}-\ref{assume:l-smoothness} optimized by SSAM. Let $\eta_t=\frac{\eta_0}{\sqrt{t}}$ and perturbation amplitude $\rho$ decay with square root of $t$, \emph{e.g.}, $\rho_t=\frac{\rho_0}{\sqrt{t}}$. With $\rho_0 \leq G\eta_0/5$, we have:
\begin{equation}
    \frac{1}{T}\sum_{t=0}^T \E||\nabla f(\w_t)||^2 \leq
    C_3\frac{1}{\sqrt{T}} + C_4\frac{\log T}{\sqrt{T}},
\end{equation}
where $C_3=\frac{2}{\eta_0}(f(\w_0)-\E f(\w_T)+L^3\eta_0^2\rho_0^2\frac{\pi^2}{6})$ and $C_4=2(L\sigma^2\eta_0+LG\rho_0)$.
\end{theorem}

For non-convex stochastic optimization, Theorems \ref{theorem:sam}\&\ref{theorem:ssam} imply that our SSAM could converge the same rate as SAM,~\emph{i.e.}, $O(\log T/\sqrt{T})$. Detailed proofs of the two theorems are given in \textbf{Supplementary Materials}.

%generalization bound
\begin{theorem}
\label{generalization}
For any $\rho>0$ and any distribution $\mathscr{D}$, with probability $1-\delta$ over the choice of the training set $\mathcal{S}\sim \mathscr{D}$, 
\begin{align}
    &L_\mathscr{D}(
    w) \leq \max_{\|\boldsymbol{\epsilon}\|_2 \leq \rho} L_\mathcal{S}(\boldsymbol{w} + \boldsymbol{\epsilon})
    + \notag
    \\
    &\sqrt{\frac{k\log\left(1+\frac{\|\boldsymbol{w}\|_2^2}{\rho^2}\left(1+\sqrt{\frac{\log(n)}{k}}\right)^2\right) + 4\log\frac{n}{\delta} + \tilde{O}(1)}{n-1}}
\end{align}
where $n=|\mathcal{S}|$, $k$ is the number of parameters and we assumed $L_\mathscr{D}(\w) \leq \mathbb{E}_{\boldsymbol{\epsilon}\odot\boldsymbol{m} \sim \mathcal{N}(0,\boldsymbol{\rho})}[L_\mathscr{D}(\w+\boldsymbol{\epsilon}\odot\boldsymbol{m})]$.
\end{theorem}

Theorem~\ref{generalization} indicates SSAM has a generalization error upbound and Detailed proof could be found in \textbf{Supplementary Materials}.

%% file: texfile/experiment.tex
\section{Experiments}
\label[section]{sec:experiments}

In this section, we evaluate the effectiveness of  SSAM through extensive experiments on CIFAR\cite{cifar10/100} and ImageNet-1K~\cite{imagenet}. The base models include conventional convolutional network,~\emph{e.g.}, VGG~\cite{vgg} , ResNet~\cite{resnet}, WideResNet~\cite{wideresnet} and vision transformer~\cite{vit, deit}. We evaluate different masking structures,~\emph{i.e.}, unstructured, structured and $N$:$M$ structured, in Sec.~\ref{sec:unstructure},  Sec.~\ref{sec:structure} and Sec.~\ref{sec:n_mstructure_experiment}, respectively.

% We report the main results on CIFAR datasets in Tables~\ref{vgg_cifar}, \ref{resnet_cifar}\&\ref{wideresnet_cifar} and ImageNet-1K  in~\cref{imagenet-experiments}. In addition to that, we also perform the noisy dataset to verify the robustness of SSAM. Then, we do the ablations and visualize the landscapes and Hessian spectra to verify that the proposed SSAM can help the model generalize better. 

\subsection{Implementation details}
\label{sec:experiments-implement-detail}
\textbf{Datasets.} 
We use CIFAR10, CIFAR100~\cite{cifar10/100} and ImageNet-1K~\cite{imagenet} as the benchmarks of our method. Specifically, CIFAR10 and CIFAR100 have 50,000 images of 32$\times$32 resolution for training, while 10,000 images for testing. ImageNet-1K~\cite{imagenet} is the most widely used benchmark for image classification, which has 1,281,167 images of 1000 classes and 50,000 images for validation. 

\textbf{Hyper-parameter setting.} 
For small resolution datasets,~\emph{i.e.}, CIFAR10 and CIFAR100, we replace the first convolution layer in ResNet and WideResNet with the one of 3$\times$3 kernel size, 1 stride and 1 padding. The models on CIFAR10/CIFAR100 are trained with 128 batch size for 200 epochs. We apply the random crop, random horizontal flip, normalization and cutout~\cite{cutout} for data augmentation, and the initial learning rate is 0.05 with a cosine learning rate schedule. The momentum and weight decay of SGD are set to 0.9 and 5\emph{e}-4, respectively. SAM and SSAM apply the same settings, except that weight decay is set to 0.001 \cite{esam}. We determine the perturbation magnitude $\rho$ from $\{0.01, 0.02, 0.05, 0.1, 0.2, 0.5\}$ via grid search. Without specifically stating, we set $\rho$ as 0.1 and 0.2 in CIFAR10 and CIFAR100, respectively. 
For ImageNet-1K, we randomly resize and crop all images to a resolution of 224$\times$224, and apply random horizontal flip, normalization  during training. We train ResNet with a batch size of 256, and adopt the cosine learning rate schedule with initial learning rate 0.1.  The momentum and weight decay of SGD is set as 0.9 and 1\emph{e}-4. SAM and SSAM use the same settings as above. The test images of both architectures are resized to 256$\times$256 and then centerly cropped to 224$\times$224. The perturbation magnitude $\rho$ is set to 0.07.

\subsection{Experiments of \emph{unstructured} masking}
\label{sec:unstructure}
In this subsection, the perturbation is masked without structural constraint,~shown in the 2-th subfigure in~\ref{mask_pattern}. 
The sparse perturbation is implemented explicitly, and the experiments of SSAM include SSAM-F and SSAM-D.

\textbf{Results on CIFAR10/CIFAR100.}\
We first evaluate our SSAM on CIFAR-10 and CIFAR100. The models we used are VGG~\cite{vgg}, ResNet-18~\cite{resnet} and WideResNet-28-10~\cite{wideresnet}. 
The perturbation magnitude $\rho$ for SAM and SSAM are the same for a fair comparison. As shown in~\cref{vgg_cifar} and~\cref{resnet_cifar}, VGG and ResNet18 with SSAM of 50\% sparsity outperform SAM of full perturbations. We can also observe that the advantages of SSAM-F and SSAM-D on WideResNet28 are more significant, which achieves better performance than SAM with up to 95\% sparsity. Note that the parameter size of WideResNet28 is much larger than that of ResNet-18 and VGG, therefore it's easy to overfit on CIFAR. In addition, even with very large sparsity, both SSAM-F and SSAM-D can still obtain competitive performance against SAM. 

\vspace{-1mm}
\begin{table}[ht]
\centering
\caption{Test accuracy of VGG11-BN on CIFAR10 with unstructured Sparse SAM($\rho$ is set to 0.5).}
\vspace{-2mm}
\label[table]{vgg_cifar}
\resizebox{\linewidth}{!}{
\begin{tabular}{cccc}
\toprule
Model & Optimizer & Sparsity & CIFAR10 \\ \hline
\multirow{8}{*}{VGG11-BN} &  SGD & / & 93.42\% \\ 
 &  SAM & 0\% & 93.87\% \\ \cline{2-4} 
 & \multirow{6}{*}{SSAM-F/SSAM-D} & 50\% & \textbf{94.03\%\textcolor{blue}{(+0.16)}}/93.79\%\textcolor{blue}{(-0.08)} \\
 &  & 80\% & 93.83\%\textcolor{blue}{(-0.04)}/\textbf{93.95\%\textcolor{blue}{(+0.08)}} \\
 &  & 90\% & 93.76\%\textcolor{blue}{(-0.11)}/93.85\%\textcolor{blue}{(-0.02)} \\
 &  & 95\% & 93.77\%\textcolor{blue}{(-0.10)}/93.48\%\textcolor{blue}{(-0.39)} \\
 &  & 98\% & 93.54\%\textcolor{blue}{(-0.43)}/93.54\%\textcolor{blue}{(-0.43)} \\
 &  & 99\% & 93.47\%\textcolor{blue}{(-0.40)}/93.33\%\textcolor{blue}{(-0.54)} \\ \bottomrule
\end{tabular}%
}
\end{table}

\textbf{Results on ImageNet.}\ 
~\cref{imagenet-experiments} reports the result of SSAM-F and SSAM-D on the large-scale ImageNet-1K~\cite{imagenet} dataset. The model we used is ResNet50~\cite{resnet}. We can observe that SSAM-F and SSAM-D can maintain the performance with 50\% sparsity. However, as sparsity ratio increases, SSAM-F will receive relatively obvious performance drops while SSAM-D is more robust. 

\vspace{-1mm}
\begin{table}[ht]
\caption{Test accuracy of ResNet-50 on ImageNet-1k with unstructured Sparse SAM.}
\vspace{-2mm}
\label[table]{imagenet-experiments}
\centering
\begin{tabular}{cccc}
\toprule
Model & Optimizer & Sparsity & ImageNet \\ \hline
\multirow{8}{*}{ResNet50} & SGD  & /  & 76.67\% \\ 
& SAM & 0\% & 77.25\% \\ \cline{2-4} 
& \multirow{3}{*}{SSAM-F (Ours)} & 50\% & \textbf{77.31\%\textcolor{blue}{(+0.06)}}  \\
& & 80\% & 76.81\%\textcolor{blue}{(-0.44)} \\
& & 90\% & 76.74\%\textcolor{blue}{(-0.51)} \\ \cline{2-4} 
& \multirow{3}{*}{SSAM-D (Ours)} & 50\% & \textbf{77.25\%\textcolor{blue}{(-0.00)}} \\
& & 80\% & 77.00\%\textcolor{blue}{(-0.25)} \\
& & 90\% & 77.00\%\textcolor{blue}{(-0.25)} \\  \bottomrule
\end{tabular}
\end{table}

\subsection{Experiments of \emph{structured} masking}
\label{sec:structure}
In this subsection, the perturbation is masked structurally, shown in the 3-th subfigure of figure~\ref{mask_pattern}. Specifically, a kernel of a convolution layer is treated as the unit to be masked. 
The sparse perturbation is implemented explicitly and implicitly, and the experiments of SSAM is based on SSAM-F, since SSAM-D is restricted since the unit to be masked has a different number of parameters.

\textbf{Results on CIFAR10/CIFAR100.} We evaluate our structure SSAM on CIFAR10 and CIFAR100, as shown in Table~\ref{tab:cifar_structure}. From this table, we can observe a considerable degradation in performance, which is believed to be the reason for the restricted mask structure. Compared to the SSAM with an unstructured mask, the performance of SSAM with a structured mask generally degrades,~\emph{e.g.}, 96.81\% \emph{v.s.} 96.73\% and 96.75\% \emph{v.s.} 96.65\%. Similar to the explicit implementation of sparse perturbation, the performance of implicit implementation degrades severely when the sparsity increase,~\emph{e.g.}, 81.02\% of CIFAR100 at 50\% sparsity and 78.87\% of CIFAR100 at 80\% sparsity. Nevertheless, implicit sparse perturbation could achieve comparable performance,~\emph{e.g.}, 81.02\% at 50\% sparsity ~\emph{v.s.} 81.03\% of SAM. 

Besides, since the unit to be masked is the kernel whose number of parameters is different, the sparsity can not be controlled precisely. The sparsity could not be set too large,~\emph{e.g.}, 99\% as used in pervious experiments.

\begin{table}[ht]
\centering
\caption{Test accuracy of ResNet-18 on CIFAR with structured Sparse SAM.}
\label{tab:cifar_structure}
\resizebox{\linewidth}{!}{
\begin{tabular}{ccccc}
\toprule
Model & Optimizer & Sparsity & CIFAR10 & CIFAR100 \\ \hline
\multirow{10}{*}{ResNet18} & SGD & / & 96.07\% & 77.80\% \\
 & SAM & 0\% & 96.83\% & 81.03\% \\ \cline{2-5} 
 & \multirow{5}{*}{SSAM-F Explicitly} & 50\% & \textbf{96.73\%\textcolor{blue}{(-0.10)}} & 80.62\%\textcolor{blue}{(-0.41)} \\
 &  & 80\% & 96.60\%\textcolor{blue}{(-0.23)} & 79.56\%\textcolor{blue}{(-1.47)} \\
 &  & 90\% & 96.65\%\textcolor{blue}{(-0.18)} & 80.14\%\textcolor{blue}{(-0.89)} \\
 &  & 95\% & 96.53\%\textcolor{blue}{(-0.30)} & 80.19\%\textcolor{blue}{(-0.84)} \\
 &  & 98\% & 96.42\%\textcolor{blue}{(-0.41)} & 79.63\%\textcolor{blue}{(-1.40)} \\ \cline{2-5} 
 & \multirow{3}{*}{SSAM-F Implicitly} & 50\% & 96.44\%\textcolor{blue}{(-0.39)} & \textbf{81.02\%\textcolor{blue}{(-0.01)}} \\
 &  & 80\% & 95.61\%\textcolor{blue}{(-1.22)} & 78.87\%\textcolor{blue}{(-2.16)} \\
 &  & 90\% & 95.97\%\textcolor{blue}{(-0.86)} & 78.66\%\textcolor{blue}{(-2.37)} \\ \bottomrule
\end{tabular}%
}
\end{table}

% \textbf{Results on ImageNet-1k.}
% Table~\ref{tab:imagenet_structure} shows the result of structured SSAM on the large-scale dataset Imagenet-1k. The conclusion is consistent with the one on CIFAR, that is, the structured masks,~\emph{i.e.}, coarse-grained sparse perturbation, have a negative impact on performance, even for Sharpness Aware Minimization.

% \begin{table}[ht]
% \centering
% \caption{Test accuracy of ResNet-50 on ImageNet-1k with structured Sparse SAM.}
% \label{tab:imagenet_structure}
% \begin{tabular}{cccc}
% \toprule
% Model & Optimizer & Sparsity & ImageNet-1k \\ \hline
% \multirow{7}{*}{ResNet50} & SGD & / & 76.67\% \\
%  & SAM & 0\% & 77.25\% \\ \cline{2-4} 
%  & \multirow{4}{*}{SSAM-F Explicitly} & 50\% & 77.17\%\textcolor{blue}{(-0.08)}  \\
%  &  & 80\% & 76.94\%\textcolor{blue}{(-0.31)} \\
%  &  & 90\% & 76.74\%\textcolor{blue}{(-0.51)} \\ 
%  \bottomrule
% \end{tabular}%
% \end{table}

\subsection{Experiments of \emph{$N$:$M$ structured} masking}
\label{sec:n_mstructure_experiment}
In this subsection, the perturbation is masked in $N$:$M$ structure, shown in the 4-th subfigure of figure~\ref{mask_pattern}. The models used in the experiments include transformer architecture, since the principal parts in transformer,~\emph{i.e.}, mlp module and attention module, are based on matrix multiplication. Besides, in order to match the matrix multiplication, we rewrite the convolution operator in \emph{img2col} mode~\cite{cnn_implement}. The experiments of SSAM is based on SSAM-F.

The \emph{cusparseLt}~\cite{cusparselt} is used for Vision Transformer to achieve truly accelerated SSAM. More details of transformer and implementation by \emph{cusparseLt} could be found in \textbf{Supplementary Materials}.

\textbf{Result on CIFAR10/CIFAR100.} The results of $N$:$M$ structured SSAM on CIFAR10 and CIFAR100 are shown in Table~\ref{tab:nm_cifar}. The performance of $N$:$M$ structured SSAM is more closer to the unstructured mask, since the structural limitation of $N$:$M$ is slacker than structured mask. The performance of implicit SSAM is more unstable and normally degrades at high sparsity,~\emph{e.g.}, 1:8.

\begin{table}[ht]
\centering
\caption{Test accuracy of ResNet-18 on CIFAR with $N$:$M$ structured Sparse SAM}
\label{tab:nm_cifar}
\resizebox{\linewidth}{!}{
\begin{tabular}{ccccc}
\hline
Model & Optimizer & N:M & CIFAR10 & CIFAR100 \\ \hline
\multirow{10}{*}{ResNet18} & SGD & / & 96.07\% & 77.80\% \\
 & SAM & / & 96.83\% & 81.03\% \\ \cline{2-5} 
 & \multirow{4}{*}{SSAM-F Explicitly} & 1:2 & \textbf{96.83\%\textcolor{blue}{(+0.00)}} & 81.14\%\textcolor{blue}{(+0.11)} \\
 &  & 1:4 & 96.70\%\textcolor{blue}{(-0.13)} & 80.68\%\textcolor{blue}{(-0.35)} \\
 &  & 1:8 & 96.53\%\textcolor{blue}{(-0.30)} & 80.72\%\textcolor{blue}{(-0.31)} \\
 &  & 2:4 & 96.79\%\textcolor{blue}{(-0.04)} & \textbf{81.25\%\textcolor{blue}{(+0.22)}} \\ \cline{2-5} 
 & \multirow{4}{*}{SSAM-F Implicitly} & 1:2 & 96.57\%\textcolor{blue}{(-0.26)} & \textbf{81.26\%\textcolor{blue}{(+0.23)}} \\
 &  & 1:4 & \textbf{96.73\%\textcolor{blue}{(-0.10)}} & 80.11\%\textcolor{blue}{(-0.92)} \\
 &  & 1:8 & 96.03\%\textcolor{blue}{(-0.80)} & 80.42\%\textcolor{blue}{(-0.61)} \\
 &  & 2:4 & 96.62\%\textcolor{blue}{(-0.11)} & 81.04\%\textcolor{blue}{(+0.01)} \\ \hline
\end{tabular}%
}
\end{table}

% \textbf{Result on ImageNet-1k}
% Due to the computational cost is not affordable, we randomly selected 20\% of the ImageNet-1k training set for training, while the whole test set for evaluation. 

\vspace{-3mm}
\begin{table}[ht]
\centering
\caption{Running time of Transformer on CIFAR of ViT with $N$:$M$ structured Sparse SAM}
\label{tab:nm_vit}
% \resizebox{\textwidth}{!}{%
\begin{tabular}{ccccc}
\toprule
Model & Optimizer & N:M & Time\tablefootnote{Time for one iteration on A100} & Speed\\ \hline
\multirow{3}{*}{ViT} & SAM & / & 4.10s & $\times$1\\ \cline{2-5} 
 & \multirow{2}{*}{SSAM-F Implicitly} & 1:2 & 3.58s & $\times$1.2\\
 &  & 2:4 & 3.53s  & $\times$1.2 \\ \bottomrule
\end{tabular}
% }
\end{table}

It should be emphasized that the experimental purpose of Sec.~\ref{tab:nm_vit} is not to compare with other SOTA algorithms, but to demonstrate that our proposed SSAM can achieve real training acceleration through CUDA libraries for relevant sparse matrices. The figure about Epoch is in the appendix.

\vspace{-2mm}
\subsection{Robustness to Noisy Data}
Since SAM converges at a flat minima suggests that the robustness of SAM to the noisy training data. To verify that without perturbation of all parameters, the Sparse SAM can also achieve comparable flat minima, we conduct the noisy datasets including label noise and corrupted images. In this subsection, we test the robustness of ResNet18 on noisy CIFAR and the sparsity of unstructured Sparse SAM is 50\%.

\textbf{Label Noise.} To be more specific, the noisy dataset is a classical noisy-label setting, in which some labels in the training set are flipped to incorrect ones, while the test set remains unaffected. We keep the experimental setting unchanged except for flipping the partial training label and without cutout augmentation. The results are presented in the following Tabel~\ref{label-noise}.

% \begin{wraptable}{l}{6cm}
\vspace{-1mm}
\begin{table}[ht]
\small
\centering
\caption{Test Accuracy on the clean CIFAR100 testset of ResNet18 trained by partional flipped label.}
\label{label-noise}
\vspace{-3mm}
\begin{tabular}{ccccc}
\toprule
\multirow{2}{*}{Method} & \multicolumn{3}{c}{Noisy Rate(\%)} \\ \cline{2-4} 
  & 5 & 10 & 20 \\ \hline
SGD  & 73.95\% & 71.08\% & 64.54\% \\
SAM & 77.74\% & 74.50\% & 69.46\% \\
SSAM-F & 77.18\% & 74.32\% & 68.91\% &  \\
SSAM-D & 76.84\% & 73.99\% & 68.78\% &  \\ \bottomrule
\end{tabular}
% \end{wraptable}
\end{table}

\textbf{Corrupted Image.} We corrupt partial training data manually to introduce noisy components via \emph{Gaussian noise}, \emph{resolution}, \emph{fog}, and \emph{central rectangle occlusion}.In particular, \emph{Gaussian Noise} applies Gaussian noise to images and \emph{Fog} indicates fogging the image. \emph{Resolution} downsamples the image and upsamples it back to its original resolution, \emph{Central rectangle occlusion} blocks the object in the middle of the image, which accounts for a quarter of an image, which are shown in figure~\ref{fig:corruption}.

\begin{figure}[ht]
    \centering
    \vspace{-0.2cm}
    \includegraphics[width=0.9\linewidth]{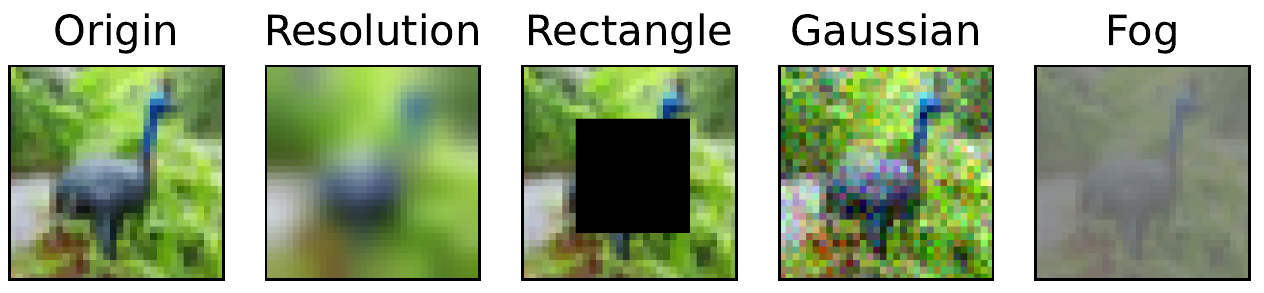}
    \vspace{-0.1cm}
    \caption{Corrupted images after \emph{resolution}, \emph{central rectangle occlusion},  \emph{Gaussian noise} or \emph{fog}. }
    \label[figure]{fig:corruption}
    \vspace{-0.4cm}
\end{figure}

As shown in Table~\ref{tab:corruption}, we test the ResNet18 trained by SAM and our Sparse SAM maintaining the same setting in Section~\ref{sec:experiments} except for cutout. From this table, the fact SAM outperforms SGD with corrupted training data indicates the robustness of SAM. In addition to this and more importantly, our Sparse SAM could maintain the superiority of SAM and even transcend it. SSAM-F has better performance with half of the perturbated parameters indicating that the perturbation of  entire parameters,~\emph{i.e.}, SAM, is unnecessary and suboptimal.

\begin{table}[ht]
\centering
\caption{Test Accuracy on the clean testset of ResNet18 trained by partional corrupted images.}
\label{tab:corruption}
\resizebox{\linewidth}{!}{%
\begin{tabular}{ccccccc}
\hline
\multicolumn{7}{c}{\cellcolor[HTML]{C0C0C0}\textbf{CIFAR10}} \\ \hline
\multirow{6}{*}{ResNet18} & & \multicolumn{5}{c}{Corrupt Rate(\%)} \\ \cline{3-7} 
&\multirow{-2}{*}{Method} & 0 & 10 & 20 & 40 & 80 \\ \cline{2-7}
&SGD & 95.36\% & 94.68\% & 94.68\% & 93.79\% & 91.16\% \\
&SAM & 96.30\% & 95.79\% & 95.48\% & 94.70\% & \textbf{92.61\%} \\
&SSAM-F & 96.36\% & \textbf{96.00\%} & 95.28\% & 94.71\% & 92.23\% \\
&SSAM-D & \textbf{96.43\%} & 95.57\% & \textbf{95.54\%} & \textbf{94.89\%} & 92.58\% \\ \hline
\multicolumn{7}{c}{\cellcolor[HTML]{C0C0C0}\textbf{CIFAR100}} \\ \hline
\multirow{6}{*}{ResNet18}& & \multicolumn{5}{c}{Corrupt Rate(\%)} \\ \cline{3-7} 
&\multirow{-2}{*}{Method} & 0 & 5 & 10 & 20 & 40 \\ \cline{2-7}
& SGD & 77.80\% & 78.01\% & 77.20\% & 75.30\% & 73.06\% \\
& SAM & \textbf{80.07\%} & 79.57\% & 78.30\% & 77.3\% & 74.77\% \\
& SSAM-F & 79.95\%&\textbf{80.07\%} & \textbf{79.04\%} &\textbf{77.75\%} & \textbf{75.01\%} \\
& SSAM-D & 79.78\% & 79.31\% & 78.38\% & 77.66\% & 74.68\% \\ \hline
\end{tabular}%
}
\end{table}

% \begin{table}[ht]
% \centering
% \caption{Test Accuracy on the clean CIFAR10 test set of ResNet18 trained by various optimizers.}
% \label{corruption:cifar10}
% \begin{tabular}{cccccc}
% \toprule
% \multirow{2}{*}{Method} & \multicolumn{5}{c}{Corrupt Rate(\%)} \\ \cline{2-6} 
%  & 0 & 10 & 20 & 40 & 80 \\ \hline
% SGD & 95.36\% & 94.68\% & 94.68\% & 93.79\% & 91.16\% \\ 
% SAM & 96.30\% & 95.79\% & 95.48\% & 94.70\% & \textbf{92.61\%} \\ 
% SSAM-F & 96.36\% & \textbf{96.00\%} & 95.28\% & 94.71\% & 92.23\% \\ 
% SSAM-D & \textbf{96.43\%} & 95.57\% & \textbf{95.54\%} & \textbf{94.89\%} & 92.58\% \\ \bottomrule
% \end{tabular}
% \end{table}

% \begin{table}[ht]
% \centering
% \caption{\small{Test Accuracy on the clean CIFAR100 test set of ResNet18 trained by various optimizers.}}
% \label{corruption:cifar100}
% \begin{tabular}{cccccc}
% \toprule
%  & \multicolumn{5}{c}{Corrupt Rate(\%)} \\ \cline{2-6} 
% \multirow{-2}{*}{Method} & 0 & 5 & 10 & 20 & 40 \\ \hline
% SGD & 77.80\% & 78.01\% & 77.20\% & 75.30\% & 73.06\% \\
% SAM & \textbf{80.07\%} & 79.57\% & 78.30\% & 77.3\% & 74.77\% \\
% SSAM-F & 79.95\% & \textbf{80.07\%} & \textbf{79.04\%} & \textbf{77.75\%} & \textbf{75.01\%} \\
% SSAM-D & 79.78\% & 79.31\% & 78.38\% & 77.66\% & 74.68\% \\ \bottomrule
% \end{tabular}
% \end{table}

\subsection{Ablation and anaysis}

\textbf{Ablations of Masking Strategy.} For further verification of our masking strategy, we perform more ablations in this paragraph. For SSAM-F, which selects the parameters with largest fisher information to be perturbed, we consider a random mask,~\emph{i.e.}, parameters are perturbed totally random. For SSAM-D, which drops the flattest weights to be masked, we conduct the experiments where SSAM-D drops randomly or drops the sharpest weights,~\emph{i.e.}, the weights with large gradients. The results of ablations are shown in~\cref{table:mask-strategy}. The results show that random strategies are less effective than our SSAM. The performance of SSAM-D dropping sharpest weights drops a lot even worse than random strategy, which is consistent with our conjecture,~\emph{i.e.} the parameters with small gradient value have reached a relatively flat position and updated by SAM doesn't bring better benefits.

\vspace{-1.5mm}
\begin{table}[htbp]
\centering
\caption{Ablation of different masking strategies.}
\label[table]{table:mask-strategy}
\vspace{-2mm}
\begin{tabular}{cccc}
\toprule
Model & Optimizer & Strategy & ImageNet \\ \hline
\multirow{7}{*}{ResNet50} & SGD & / & 76.67\% \\
& SAM & / & 77.25\% \\ \cline{2-4}
& \multirow{2}{*}{SSAM-F} & Random Mask & 77.08\%\textcolor{blue}{(-0.17)} \\
&  & Topk Fisher Information & 77.31\%\textcolor{blue}{(+0.06)} \\ \cline{2-4}
& \multirow{3}{*}{SSAM-D} & Random Drop & 77.08\%\textcolor{blue}{(-0.17)} \\
& & Drop Sharpnest weights & 76.68\%\textcolor{blue}{(-0.57)} \\
& & Drop Flattest weights & 77.25\%\textcolor{blue}{(-0.00)} \\ \bottomrule
\end{tabular}
\end{table}

\vspace{-1.5mm}
\textbf{Gradient ascent with few data.} As stated in~\cite{smallbsz_ga_sam}, an efficient strategy to reduce computation cost is using a subset,~\emph{e.g.} 20\%, of the minibatch for the ascent gradient calculation,~\emph{i.e.} finding the perturbation. We also perform the experiments of SAM with a subset of minibatch for calculating perturbation and keep else unchanged. We show the result of SSAM with 50\% sparsity in Table~\ref{tab:smallbsz4ga}. Contrary to the origin, the strategy does not bring a gain to SAM, but a decrease. The reason we believe is the gap between dataset, origin conducts experiment on language dataset and ours is on image dataset. However, our SSAM still has better performance compared with SAM in such case.
\begin{table}[ht]
\centering
\caption{Results of ResNet18 on CIFAR10 with different data size for calculating perturbation.}
\label{tab:smallbsz4ga}
\begin{tabular}{ccccc}
\toprule
Model & Data & Optimizer & CIFAR10 & Time(Speed up)\tablefootnote{Time for one epoch on A100}  \\ \hline
\multirow{6}{*}{ResNet18} & \multirow{3}{*}{100\%} & SAM & 96.83\% & 13.60s($\times$1) \\ 
 &  & SSAM-F & 96.81\% & 15.55s($\times$0.87) \\ 
 &  & SSAM-D & 96.87\% & 14.67s($\times$0.92) \\ \cline{2-5} 
 & \multirow{3}{*}{25\%} & SAM & 96.55\% & 12.56s($\times$1.08) \\ 
 &  & SSAM-F & 96.57\% & 13.38s($\times$1.01) \\  
 &  & SSAM-D & 96.66\% & 13.47s($\times$1.00) \\ \bottomrule
\end{tabular}
\end{table}

\textbf{Training curves.}
We visualize the training curves of SGD, SAM and SSAM in~\cref{fig:train-curve}. The training curves of SGD are more jitter, while SAM is more stable. In SSAM, half of gradient updates are the same as SGD, but its training curves are similar to those of SAM, suggesting the effectiveness. 
The trend of the training curve shows that SSAM has the same convergence rate as SAM, which supports the validity of the convergence analysis in Sec~\ref{sec:method:theoretical}.

% \vspace{-3mm}
\begin{figure}[ht]
\centering
\subfloat[ResNet\&CIFAR10]{
    \centering
    \includegraphics[width=0.30\linewidth]{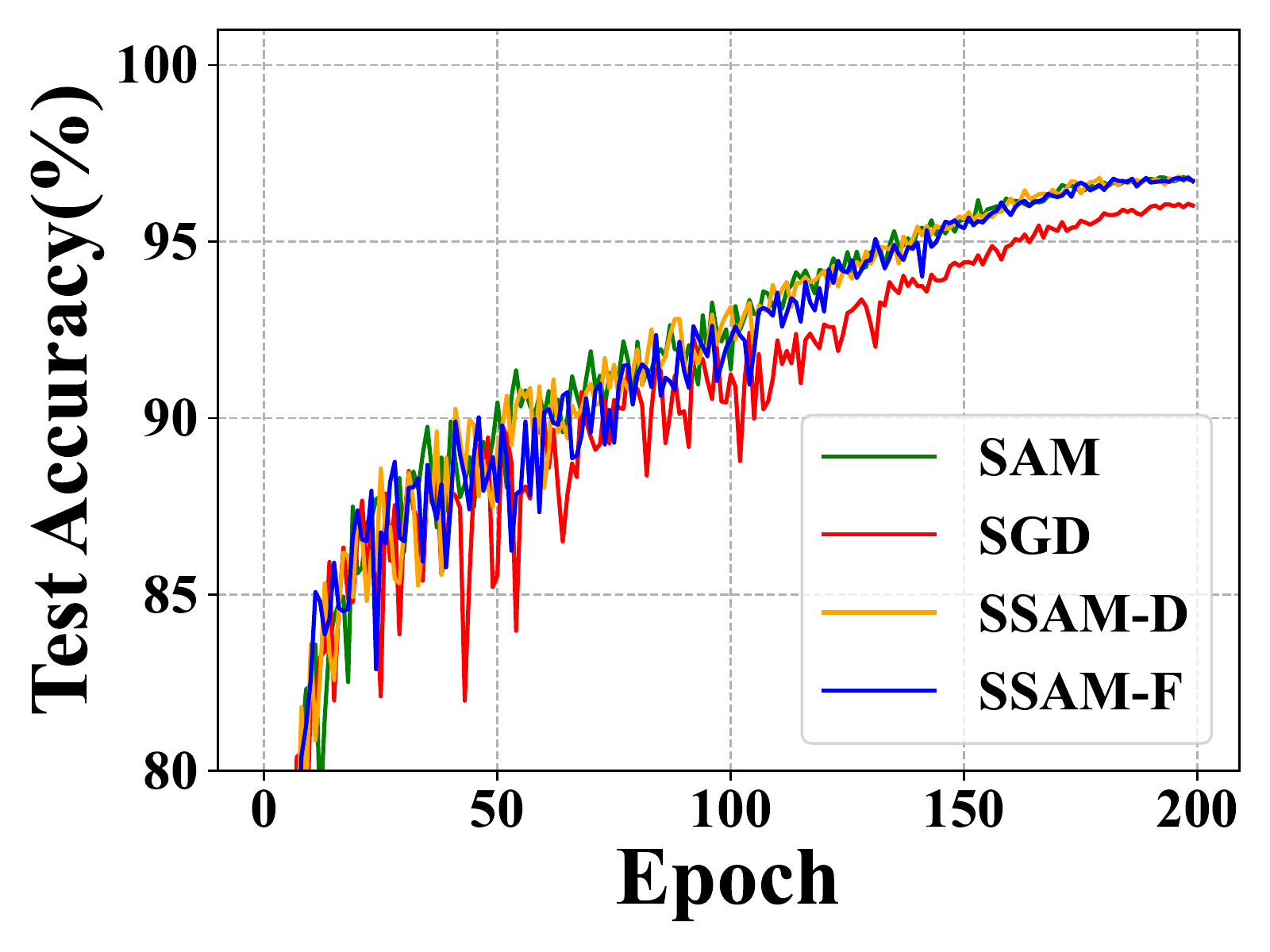}
    \label{fig:resnet18-cifar10-train-cruve}
}
\hspace{-1mm}
\subfloat[ResNet\&CIFAR100]{
    \centering
    \includegraphics[width=0.30\linewidth]{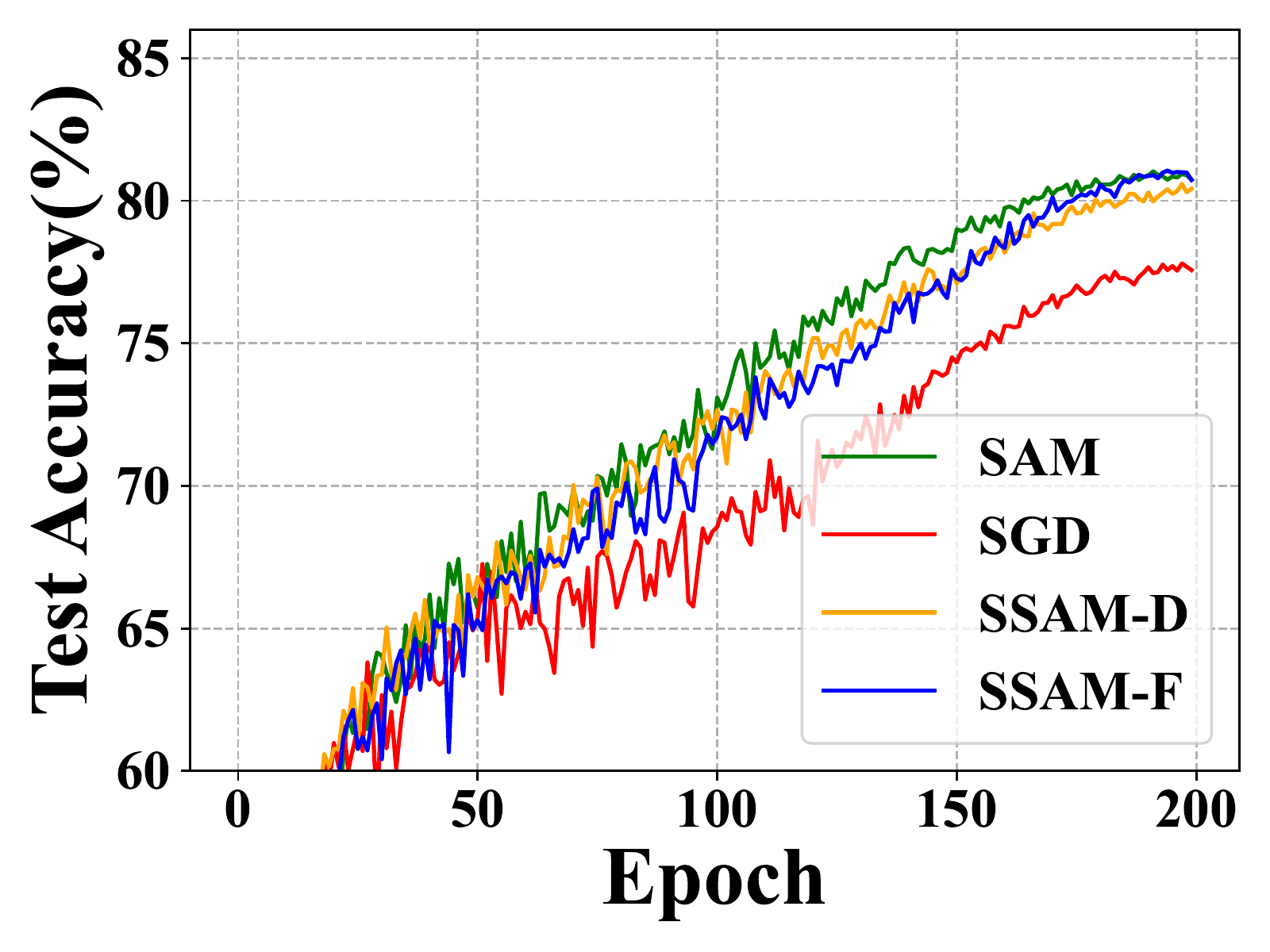}
    \label{fig:resnet18-cifar100-train-cruve}
}
\hspace{-1mm}
\subfloat[ResNet\&ImageNet]{
    \centering
    \includegraphics[width=0.30\linewidth]{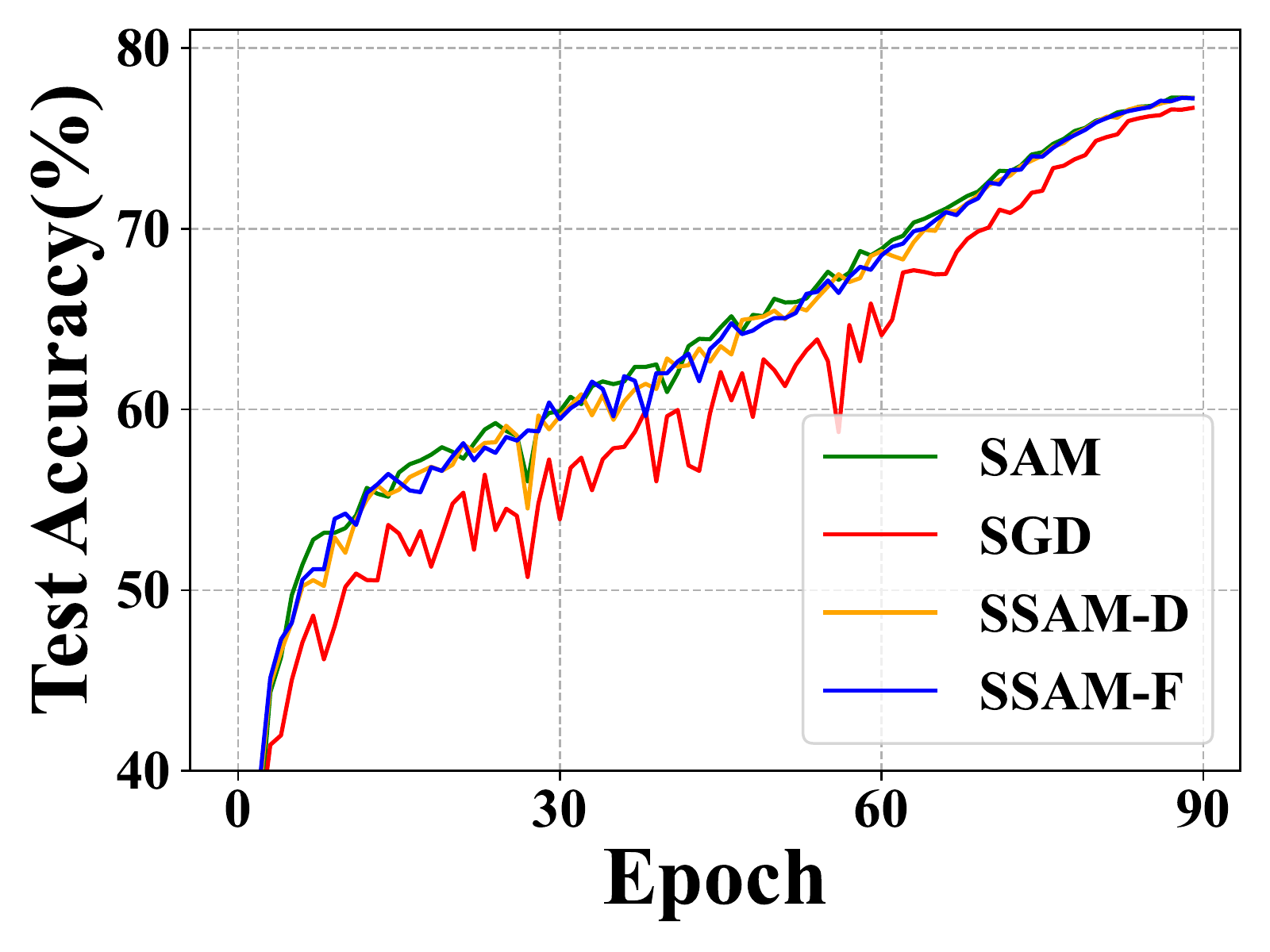}
    \label{fig:resnet50-imagenet-train-cruve}
}
\vspace{-2mm}
\caption{The training curves of SGD, SAM and SSAM. The sparsity of SSAM is 50\%.}
\label[figure]{fig:train-curve}
\end{figure}

\textbf{Sparsity \emph{vs.} Accuracy.}
We report the effect of sparsity ratios in SSAM, as depicted in~\cref{fig:sparsity-acc}. We can observe that on CIFAR datasets, the sparsities of SSAM-F and SSAM-D pose little impact on performance. In addition, they can obtain better accuracies than SAM with up to 99\% sparsity. On the much larger dataset,~\emph{i.e.}, ImageNet, a higher sparsity will lead to more obvious performance drop. 

\begin{figure}[ht]
\centering
\subfloat[WRNet\&CIFAR10]{
    \centering
    \!\!\!
    \includegraphics[width=0.32\linewidth]{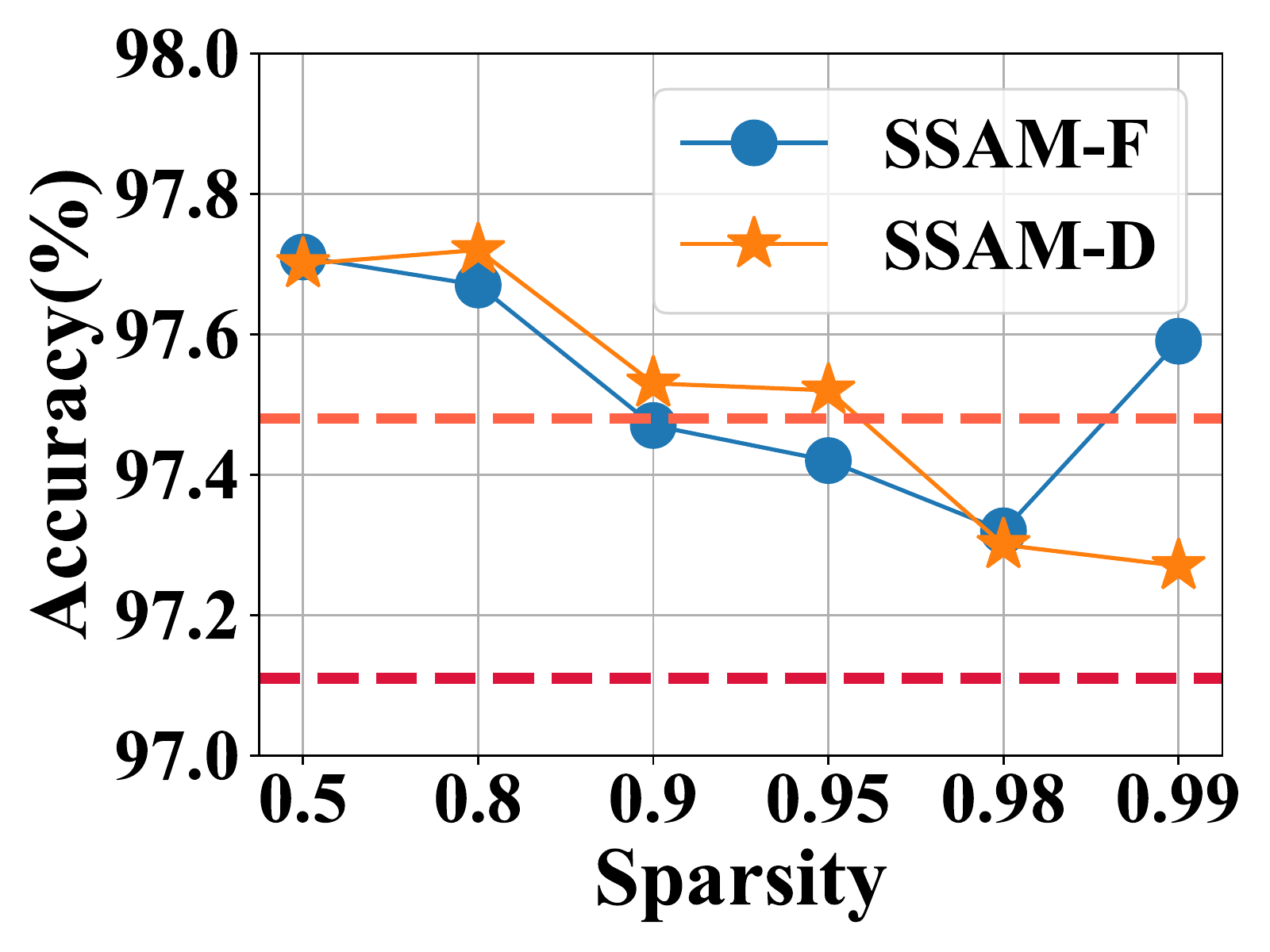}
}
% \hspace{+0.1mm}
\subfloat[WRNet\&CIFAR100]{
    \centering
    \includegraphics[width=0.32\linewidth]{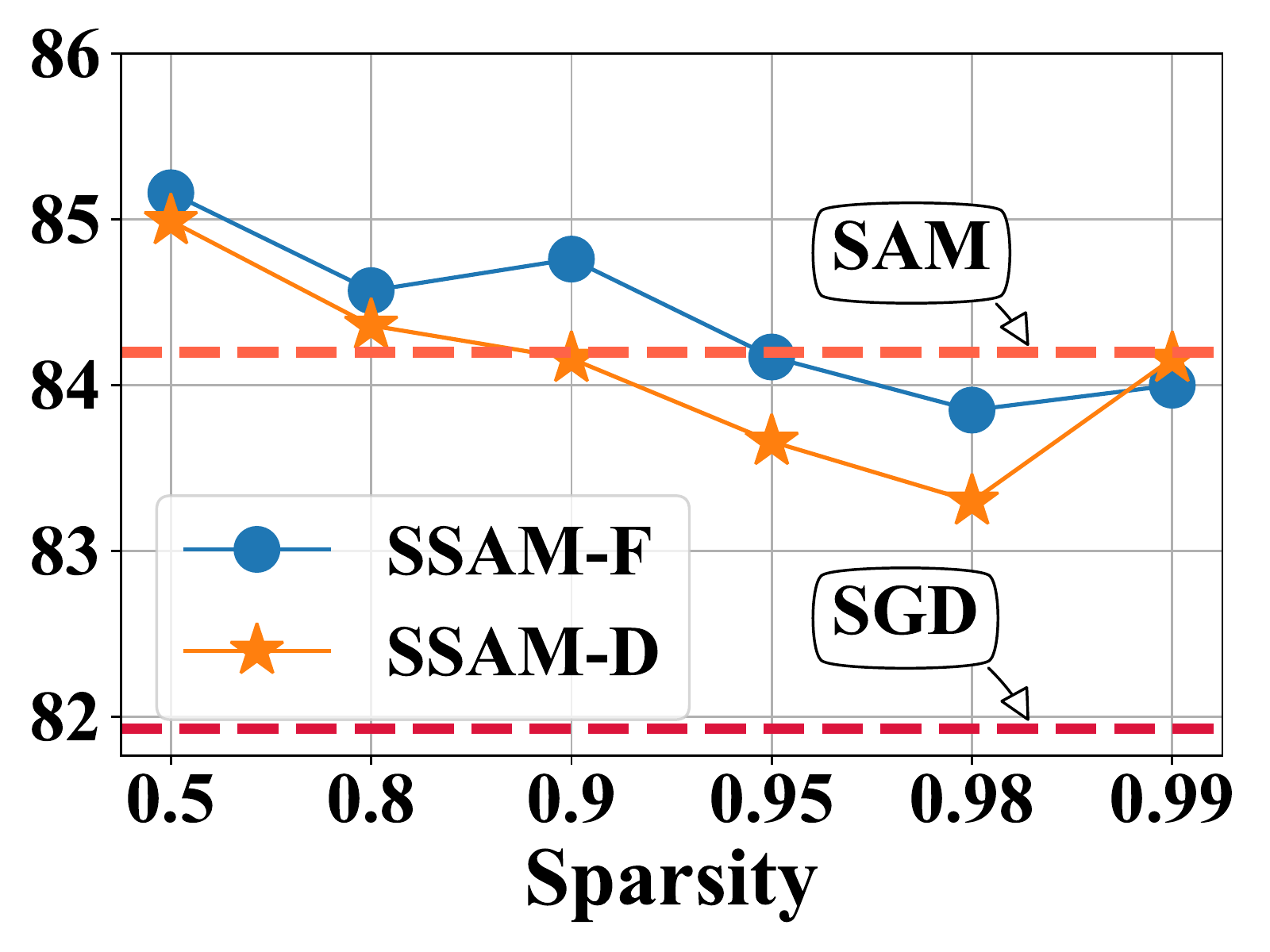}
}
% \hspace{+0.1mm}
\subfloat[ResNet\&ImageNet]{
    \centering
    \includegraphics[width=0.30\linewidth]{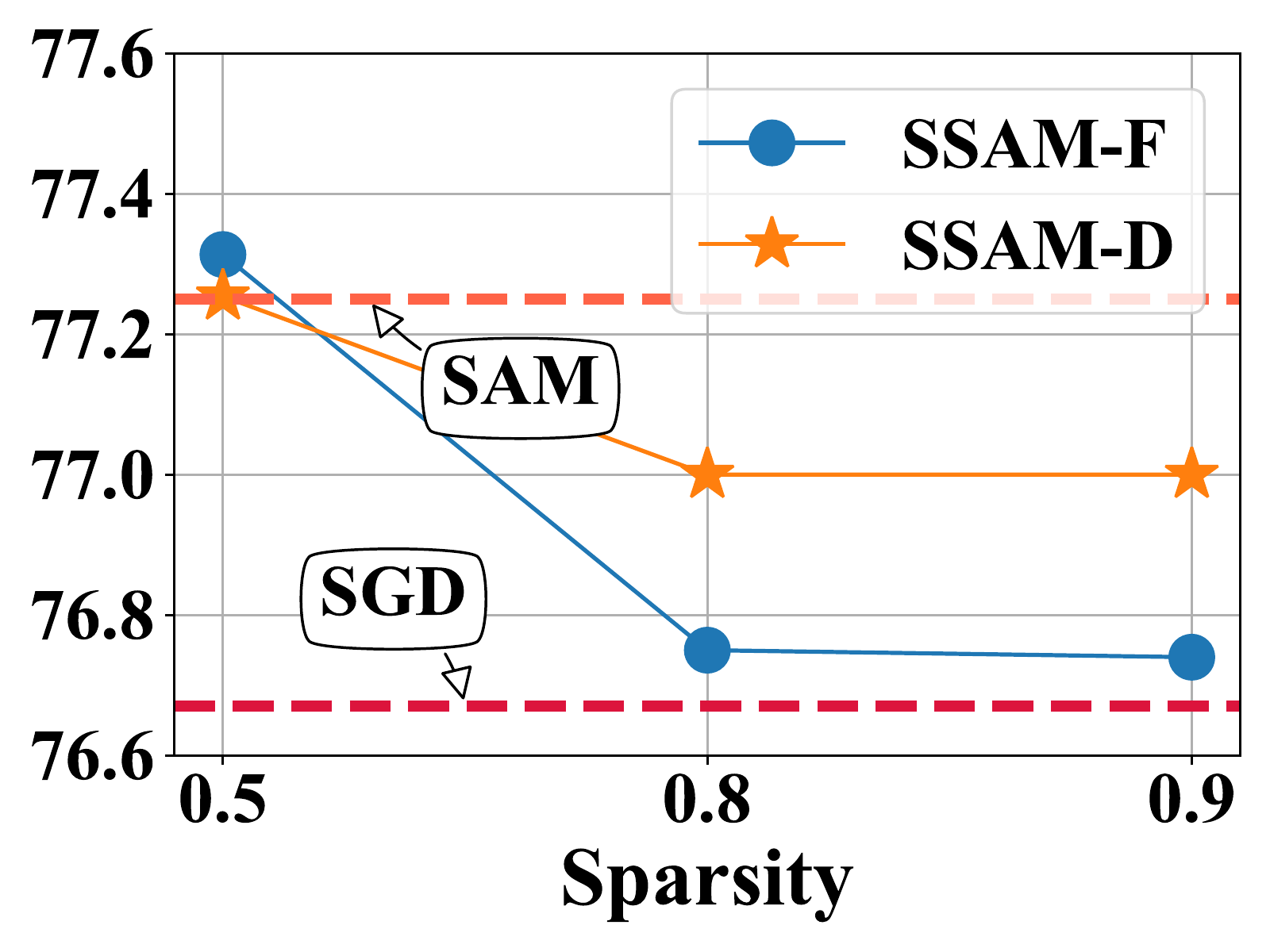}
    }
\vspace{-2mm}
\caption{Accuarcy \emph{v.s.} sparsity on CIFAR and ImageNet, where WRNet is the abbreviation of WideResNet.}
\label[figure]{fig:sparsity-acc}
\end{figure}

\subsection{SSAM with better generalization}
\textbf{Visualization of landscape.} For a more intuitive comparison between different optimization schemes, we perform a qualitative analysis for generalization,~\emph{i.e.}, visualize the training loss landscapes of ResNet18 optimized by SGD, SAM and SSAM as shown in~\cref{fig:vis-landscape}. Following~\cite{vis-landscape}, we sample $50\times50$ points in the range of $[-1,1]$ from random "filter normalized"~\cite{vis-landscape} directions, \emph{i.e.}, the $x$ and $y$ axes. As shown in~\cref{fig:vis-landscape}, the landscape of SSAM is flatter than both SGD and SAM, and most of its area is low loss (blue). This result indicates that SSAM can smooth the loss landscape notably with sparse perturbation, and it also suggests that the complete perturbation on all parameters will result in suboptimal minima.

\begin{figure}[h]
    \centering
    \includegraphics[width=1\linewidth]{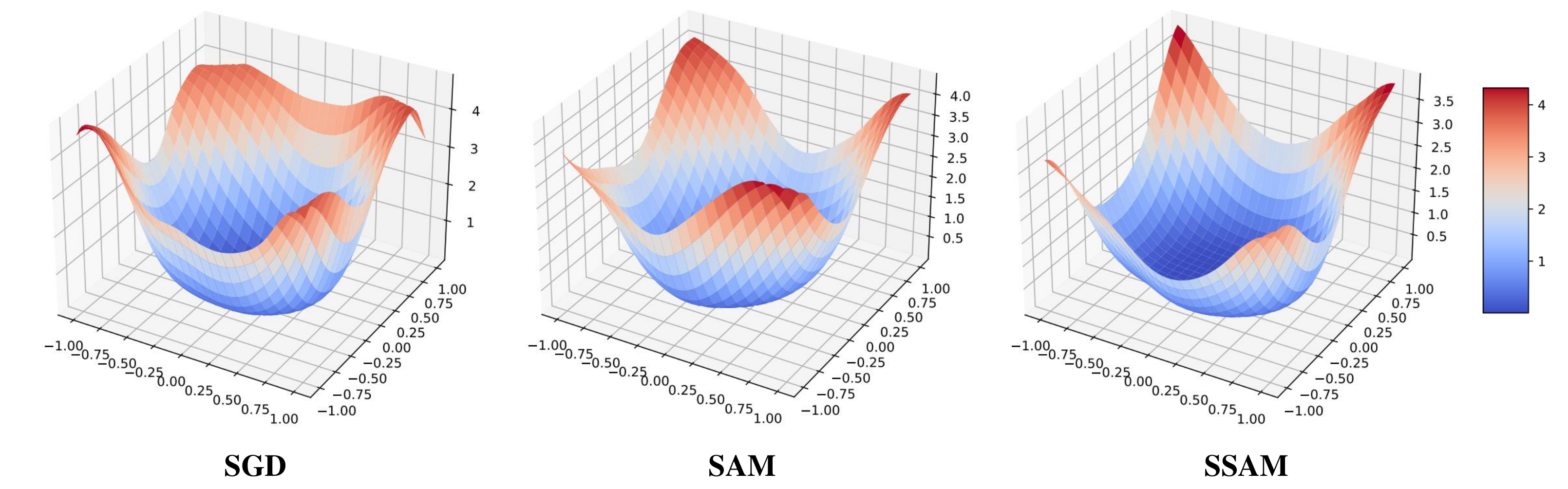}
    \caption{Training loss landscapes of ResNet18 on CIFAR10 trained with SGD, SAM, SSAM.}
    \label{fig:vis-landscape}
\end{figure}

\textbf{Hessian spectra.}
To further verify the generalization ability, we also perform quantitative analysis. In~\cref{fig:hessian}, we report the Hessian spectrum to demonstrate that SSAM can converge to a flat minima. Here, we also report the ratio of dominant eigenvalue to fifth largest ones,~\emph{i.e.}, 
$\lambda_1/\lambda_5$, used as the criteria in~\cite{sam, lambda1/lambda5}. Note that the size of Hessian is the square of the number of model parameters, which makes the eigenvalues of Hessian impossible to be computed. 
We approximate the Hessian spectrum by the Lanczos algorithm~\cite{appro-hessian} and illustrate the Hessian spectra of ResNet18 using SGD, SAM and SSAM on CIFAR10. From this figure, we observe that the dominant eigenvalue $\lambda_1$ with SSAM is less than SGD and comparable to SAM. It confirms that SSAM with a sparse perturbation can still converge to the flat minima as SAM does, or even better.

\begin{figure}[ht]
\centering
\subfloat[SGD H-spectrum]{
    \centering
    \includegraphics[width=0.30\linewidth]{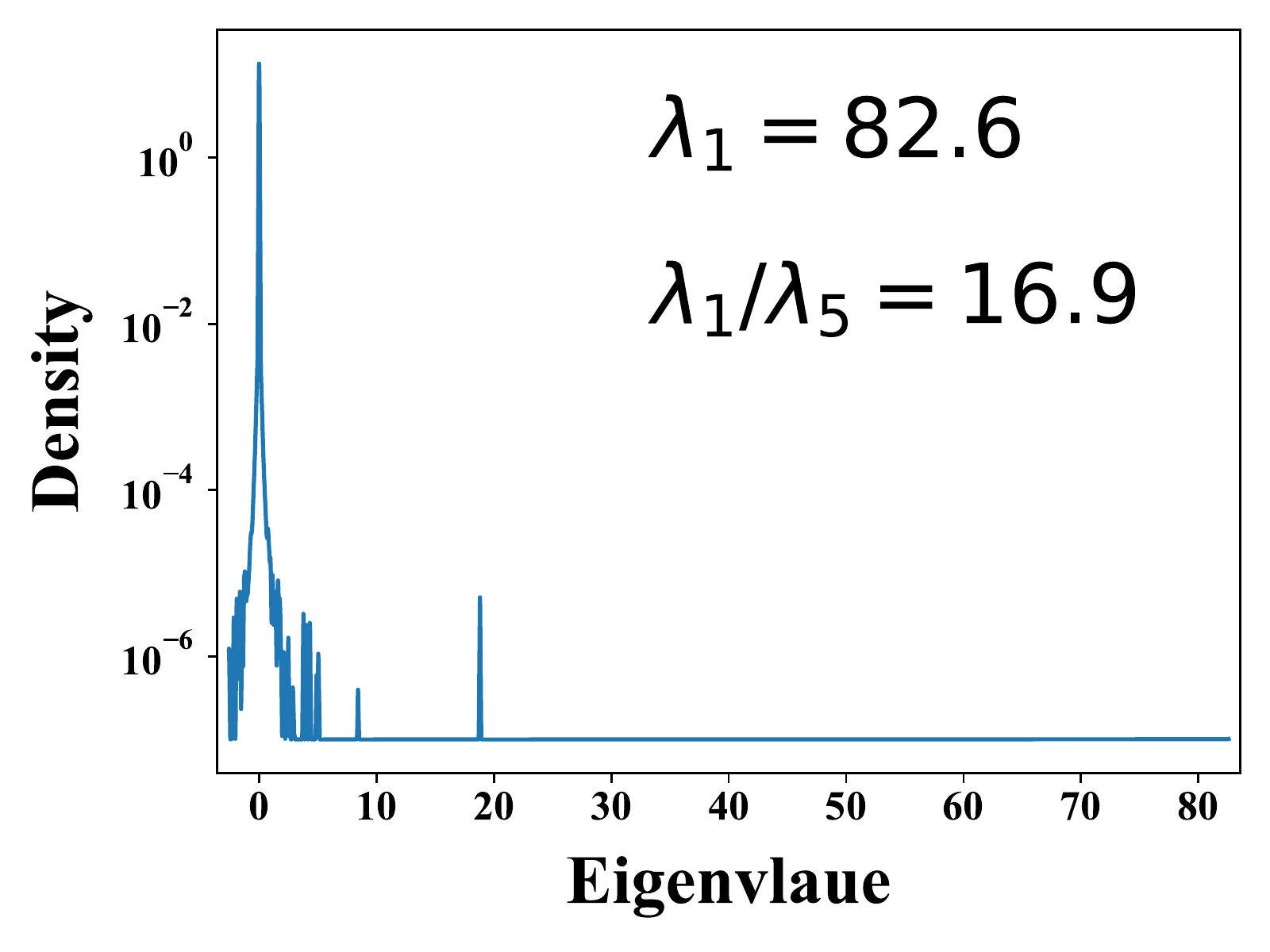}
    \label{fig:hessian-sgd}
}
\hspace{-1mm}
\subfloat[SAM H-spectrum]{
    \centering
    \includegraphics[width=0.30\linewidth]{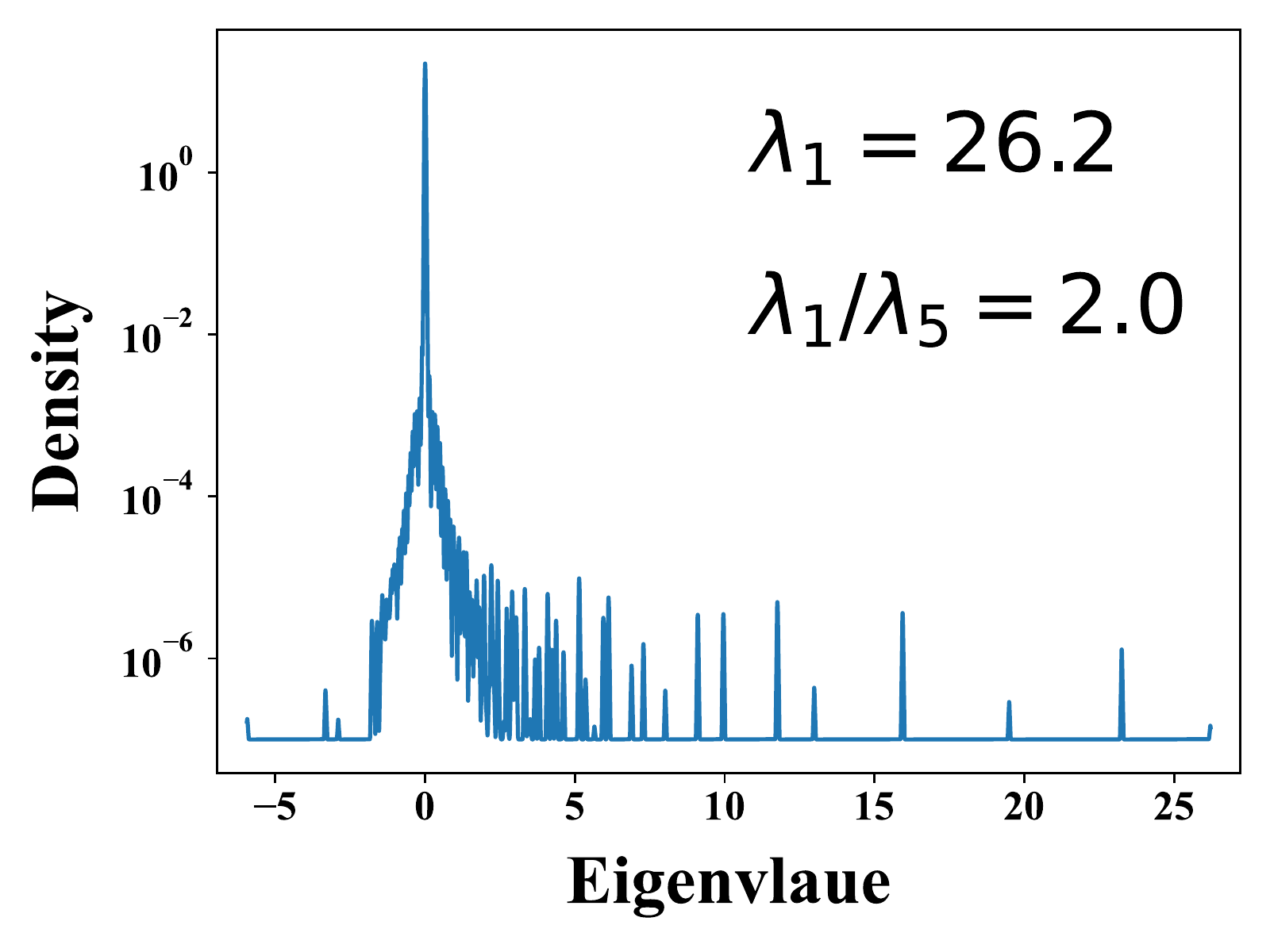}
    \label{fig:hessian-sam}
}
\hspace{-1mm}
\subfloat[SSAM H-spectrum]{
    \centering
    \includegraphics[width=0.30\linewidth]{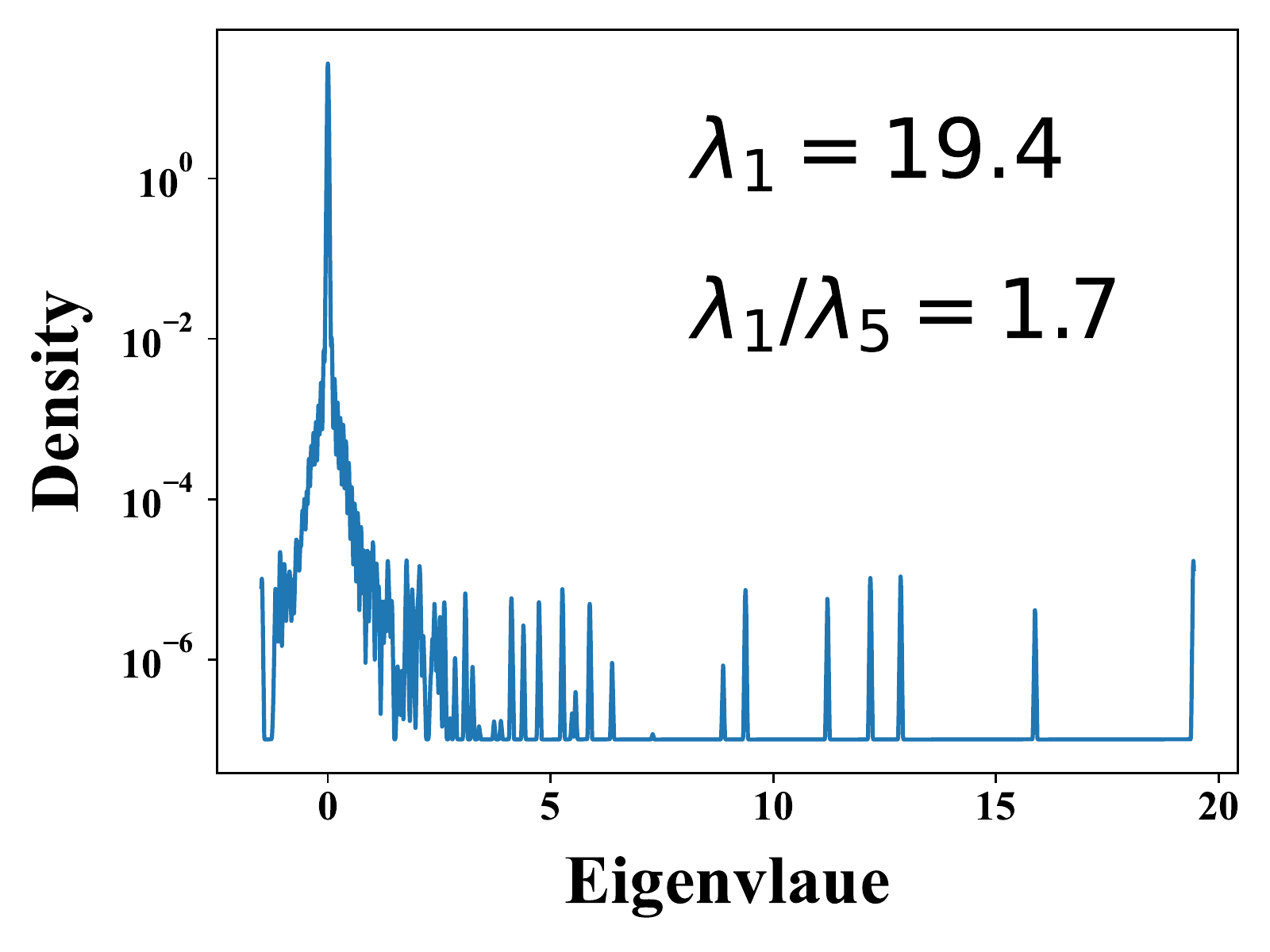}
    \label{fig:hessian-ssam}
}
\vspace{-2mm}
\caption{Hessian spectra of ResNet18 using SGD, SAM and SSAM on CIFAR10, where H-spectrum is the abbreviation of Hessian spectra.}
\label{fig:hessian}
\end{figure}

%% file: texfile/conclusion.tex
\section{Conclusion}
In this paper, we reveal that the SAM gradients for most parameters are not significantly different from the SGD ones. Based on this finding, we propose an efficient training scheme called Sparse SAM, which is achieved by computing a sparse perturbation. For a comprehensive investigation, we implement masks with different patterns,~\emph{e.g.}, unstructured, structured, and $N$:$M$ structured, and we implement sparse perturbation implicitly and explicitly. We provide two solutions for sparse perturbation, which are based on Fisher information and dynamic sparse training, respectively.  In addition, we also theoretically prove that SSAM has the same convergence rate as SAM. We validate our SSAM on extensive datasets with various models. The experimental results show that while retaining much better efficiency, SSAM can achieve competitive and even better performance than SAM.

\vspace{-3mm}
\section*{Acknowledgement}
\small{
This work is supported by the Major Science and Technology Innovation 2030 “Brain Science and Brain-like Research” key project (No. 2021ZD0201405), the National Science Fund for Distinguished Young Scholars (No.
62025603), the National Natural Science Foundation of China (No. U21B2037, No. 62176222,
No. 62176223, No. 62176226, No. 62072386, No. 62072387, No. 62072389, and No. 62002305),
Guangdong Basic and Applied Basic Research Foundation (No. 2019B1515120049), and the Natural
Science Foundation of Fujian Province of China (No. 2021J01002).}

%% file: texfile/appendix.tex
\onecolumn

\begin{appendices}

\begin{center}
    \Large \bf Supplementary Materials
\end{center}

\newenvironment{assumption-multiplex}[1]
  {\renewcommand{\theassumption}{\ref{#1}$'$}%
   \addtocounter{assumption}{-1}%
   \begin{assumption}}
  {\end{assumption}}

\section{Missing Proofs}

In this section, we will proof the main theorem in our paper. Before the detailed proof, we first recall the following assumptions that are commonly used for characterizing the convergence of non-convex stochastic optimization.
\begin{assumption-multiplex}{assume:bounded-gradient}
%\label{assume:bounded-gradient}
(Bounded Gradient.) It exists $G \geq 0$ s.t. $||\nabla f(\w)|| \leq G$.
\end{assumption-multiplex}

\begin{assumption-multiplex}{assume:bounded-variance}
%\label{assume:bounded-variance}
(Bounded Variance.) It exists $\sigma \geq 0$ s.t. $\E [||g(\w) - \nabla f(\w)||^2] \leq \sigma^2$.
\end{assumption-multiplex}

\begin{assumption-multiplex}{assume:l-smoothness}
%\label{assume:l-smoothness}
(L-smoothness.) It exists $L > 0 $ s.t. $||\nabla f(\w) - \nabla f(\boldsymbol{v})|| \leq L||\w - \boldsymbol{v}|| $, $\forall  \w, \boldsymbol{v} \in \mathbb{R}^d$.
\end{assumption-multiplex}

\subsection{Proof of Theorem 1}
Based on the objective function of Sharpness-Aware Minimization~(SAM), suppose we can obtain the noisy observation gradient $g(\w)$ of true gradient $\nabla f(\w)$, we can write the iteration of SAM:
\begin{equation}
\label[equation]{equ:iter-sam}
\left\{
    \begin{array}{ll}
        \w_{t+\frac{1}{2}}=&\w_t + \rho \cdot \frac{g(\w_t)}{||g(\w_t)||} \\
        \w_{t+1}=&\w_t - \eta \cdot g(\w_{t+\frac{1}{2}})
    \end{array}
\right.
\end{equation}

%%%%%%%%%%%%%%%%%%%%%%%% Lemma 1 %%%%%%%%%%%%%%%%%%%%%%%%
\begin{lemma}
\label{lemma:1}
For any $\rho > 0$, $L > 0$ and the differentiable function $f$, we have the following inequality:
\begin{align*}
    \langle \nabla f(\w_{t}), \nabla f(\w_t + \rho \frac{\nabla f(\w_t)}{||\nabla f(\w_t)||})\rangle 
    \geq 
    ||\nabla f(\w_t)||^2 - \rho L G
\end{align*}
\end{lemma}

\begin{proof}
We first add and subtract a term $||\nabla f(\w_t)||$ to make use of classical inequalities bounding $\langle \nabla f(\w_1) - \nabla f(\w_2), \w_1 - \w_2 \rangle$ by $||\w_1 - \w_2||^2$ for smooth.

\begin{align*}
LHS
=& \langle \nabla f(\w_t) ,\nabla f(\w_t + \rho \frac{\nabla f(\w_t)}{||\nabla f(\w_t)||}) - \nabla f(\w_t) \rangle + ||\nabla f(\w_t)||^2
\\
=& \frac{||\nabla f(\w_t)||}{\rho} \langle \frac{\rho}{||\nabla f(\w_t)||}\nabla f(\w_t) ,\nabla f(\w_t + \rho \frac{\nabla f(\w_t)}{||\nabla f(\w_t)||})) - \nabla f(\w_t) \rangle + ||\nabla f(\w_t)||^2
\\
\geq&  -L \frac{||\nabla f(\w_t)||}{\rho} \bigl|\bigl|\frac{\rho}{||\nabla f(\w_t)||}\nabla f(\w_t)\bigl|\bigl|^2 + ||\nabla f(\w_t)||^2
\\
=& - ||\nabla f(\w_t)|| \rho L + ||\nabla f(\w_t)||^2
\\
\geq& - G \rho L + ||\nabla f(\w_t)||^2
\end{align*}

where the first inequality is that
\begin{align*}
    \langle \nabla f(\w_1)-\nabla f(\w_2), \w_1 - \w_2 \rangle \geq -L||\w_1 - \w_2||^2,
\end{align*}
and the second inequality is the Assumption~\ref{assume:bounded-gradient}.
\end{proof}

%%%%%%%%%%%%%%%%%%%%%%%% Lemma 2 %%%%%%%%%%%%%%%%%%%%%%%%
\begin{lemma}
\label{lemma:2}
For $\rho > 0$, $L>0$, the iteration~\ref{equ:iter-sam} satisfies following inequality:
\begin{align*}
    \E\langle \nabla f(\w_t), g(\w_{t + \frac{1}{2}})\rangle 
    \geq
    \frac{1}{2}||\nabla f(\w_t)||^2 - 2L^2 \rho^2 - L \rho G
\end{align*}
\end{lemma}

\begin{proof}
We denote the deterministic values of $\w_{t+\frac{1}{2}}$ as $\hat{\w}_{t+\frac{1}{2}} = \w_t + \rho \frac{\nabla f(\w_t)}{||\nabla f(\w_t)||}$ in this section. After we add and subtract the term $g(\hat{\w}_{t+\frac{1}{2}})$, we have the following equation:
\begin{align*}
    \langle \nabla f(\w_t), g(\w_{t + \frac{1}{2}})\rangle 
    =&
    \langle \nabla f(\w_t), g(\w_t + \rho \frac{g(\w_t)}{||g(\w_t)||}) - g(\hat{\w}_{t+\frac{1}{2}})\rangle  + \langle g(\hat{\w}_{t+\frac{1}{2}}), \nabla f(\w_t) \rangle
\end{align*}
For the first term, we bound it by using the smoothness of $g(\w)$:
\begin{align*}
    - \langle \nabla f(\w_t), g(\w_t + \rho \frac{g(\w_t)}{||g(\w_t)||}) - g(\hat{\w}_{t+\frac{1}{2}})\rangle
    \leq& \frac{1}{2}||g(\w_t + \rho \frac{g(\w_t)}{||g(\w_t)||}) - g(\hat{\w}_{t+\frac{1}{2}})||^2 + \frac{1}{2}||\nabla f(\w_t)||^2
    \\
    \leq& \frac{L^2}{2} ||\w_t + \rho \frac{g(\w_t)}{||g(\w_t)||} - \hat{\w}_{t+\frac{1}{2}}||^2 + \frac{1}{2}||\nabla f(\w_t)||^2
    \\
    =& \frac{L^2}{2} ||\w_t + \rho \frac{g(\w_t)}{||g(\w_t)||} - (\w_t+\rho \frac{\nabla f(\w_t)}{||\nabla f(\w_t)||})||^2 + \frac{1}{2}||\nabla f(\w_t)||^2
    \\
    =& \frac{L^2\rho^2}{2} ||\frac{g(\w_t)}{||g(\w_t)||} - \frac{\nabla f(\w_t)}{||\nabla f(\w_t)||}||^2 + \frac{1}{2}||\nabla f(\w_t)||^2
    \\
    \leq& 2L^2 \rho^2 + \frac{1}{2}||\nabla f(\w_t)||^2
\end{align*}
For the second term, by using the Lemma~\ref{lemma:1}, we have:
\begin{align*}
    \E\langle g(\hat{\w}_{t+\frac{1}{2}}), \nabla f(\w_t) \rangle
    =&
    \langle \nabla f(\hat{\w}_{t+\frac{1}{2}}), \nabla f(\w_t) \rangle
    \\
    =& \langle \nabla f(\w_t + \rho \frac{\nabla f(\w_t)}{||\nabla f(\w_t)||}), \nabla f(\w_t) \rangle
    \\
    \geq& ||\nabla f(\w_t)||^2 - \rho L G
\end{align*}
Asesmbling the two inequalities yields to the result. 
\end{proof}

%%%%%%%%%%%%%%%%%%%%%%% lemma 3 %%%%%%%%%%%%%%%%%%%%%%%%%%%%%%%%
\begin{lemma}
\label{lemma:3}
For $\eta \leq \frac{1}{L}$, the iteration~\ref{equ:iter-sam} satisfies for all $t>0$:
\begin{align*}
    \E f(\w_{t+1}) \leq \E f(\w_t) - \frac{\eta}{2}\E||\nabla f(\w_t)||^2 + L\eta^2\sigma^2 +(2-L\eta)\eta L^2\rho^2 + (1-L\eta)\eta LG\rho
\end{align*}
\end{lemma}

\begin{proof}
By the smoothness of the function $f$, we obtain
\begin{align*}
    f(\w_{t+1}) 
    \leq& 
    f(\w_t) - \eta \langle \nabla f(\w_t), g(\w_{t+\frac{1}{2}}) \rangle + \frac{L\eta^2}{2}||g(\w_{t+\frac{1}{2}})||^2
    \\
    =& f(\w_t) - \eta \langle \nabla f(\w_t), g(\w_{t+\frac{1}{2}}) \rangle + \frac{L\eta^2}{2} ( ||\nabla f(\w_t)-g(\w_{t+\frac{1}{2}})||^2-||\nabla f(\w_t)||^2 + 2\langle \nabla f(\w_t), g(\w_{t+\frac{1}{2}})\rangle)
    \\
    =& f(\w_t) -\frac{L\eta^2}{2}||\nabla f(\w_t)||^2 + \frac{L\eta^2}{2} ||\nabla f(\w_t) - g(\w_{t+\frac{1}{2}})||^2 - (1-L\eta)\eta \langle  \nabla f(\w_t), g(\w_{t+\frac{1}{2}})\rangle
    \\
    \leq& f(\w_t) - \frac{L\eta^2}{2}||\nabla f(\w_t)||^2 + L\eta^2 ||\nabla f(\w_t) - g(\w_t)||^2 + L\eta^2 ||g(\w_t) - g(\w_{t+\frac{1}{2}})||^2 \\& - (1-L\eta)\eta \langle  \nabla f(\w_t), g(\w_{t+\frac{1}{2}})\rangle
    \\
    \leq& f(\w_t) - \frac{L\eta^2}{2}||\nabla f(\w_t)||^2 + L\eta^2 ||\nabla f(\w_t) - g(\w_t)||^2 + L\eta^2 L^2||\w_t - \w_{t+\frac{1}{2}}||^2 \\& - (1-L\eta)\eta \langle  \nabla f(\w_t), g(\w_{t+\frac{1}{2}})\rangle
    \\
    =& f(\w_t) - \frac{L\eta^2}{2}||\nabla f(\w_t)||^2 + L\eta^2 ||\nabla f(\w_t) - g(\w_t)||^2 + \eta^2L^3 \rho^2
    \\& -  (1-L\eta)\eta \langle  \nabla f(\w_t), g(\w_{t+\frac{1}{2}})\rangle
\end{align*}
Taking the expectation and using Lemma~\ref{lemma:2} we obtain
\begin{align*}
\E f(\w_{t+1}) 
\leq& 
\E f(\w_t) - \frac{L\eta^2}{2}\E ||\nabla f(\w_t)||^2 + L\eta^2\E ||\nabla f(\w_t) - g(\w_t)||^2 + \eta^2L^3 \rho^2 \\& - (1-L\eta)\eta \E \langle  \nabla f(\w_t), g(\w_{t+\frac{1}{2}})\rangle
\\
\leq& \E f(\w_t) - \frac{L\eta^2}{2}\E ||\nabla f(\w_t)||^2 + L\eta^2 \sigma^2 + \eta^2L^3 \rho^2 \\& - (1-L\eta)\eta \E \langle  \nabla f(\w_t), g(\w_{t+\frac{1}{2}})\rangle
\\
\leq& \E f(\w_t) - \frac{L\eta^2}{2}\E ||\nabla f(\w_t)||^2 + L\eta^2 \sigma^2 + \eta^2L^3 \rho^2 \\& - (1-L\eta)\eta \left[\frac{1}{2}\E||\nabla f(\w_t)||^2 - 2L^2\rho^2 -L\rho G\right]
\\
=& \E f(\w_t) - \frac{\eta}{2} \E ||\nabla f(\w_t)||^2 + L\eta^2 \sigma^2 + \eta^2L^3 \rho^2 + 2(1-L\eta)\eta  L^2\rho^2 + (1-L\eta)\eta L\rho G
\\
=& \E f(\w_t) - \frac{\eta}{2} \E ||\nabla f(\w_t)||^2 + L\eta^2 \sigma^2 + (2-L\eta)\eta L^2\rho^2 + (1-L\eta)\eta L G \rho
\end{align*}
\end{proof}

%%%%%%%%%%%%%%%%%%%%%%%%%%%%%%%%%%%%%%%%%%%%%%%%%%%%%%%%%%%%%%%%%%%%%
\begin{proposition}
\label[Proposition]{pro:1}
Let $\eta_t = \frac{\eta_0}{\sqrt{t}}$ and perturbation amplitude $\rho$ decay with square root of $t$, \emph{e.g.}, $\rho_t=\frac{\rho_0}{\sqrt{t}}$. For $\rho_0 \leq \frac{1}{2}G \eta_0$ and $\eta_0 \leq \frac{1}{L}$, we have

\begin{align*}
    \frac{1}{T} \sum_{t=1}^T \E ||\nabla f(\w_t)||^2 \leq& C_1 \frac{1}{\sqrt{T}} + C_2\frac{\log T}{\sqrt{T}},
\end{align*}
where $C_1=\frac{2}{\eta_0}(f(\w_0)-\E f(\w_T))$ and $C_2=2 (L \sigma^2 \eta_0 + LG\rho_0)$.

\begin{proof}
By Lemma~\ref{lemma:3}, we replace $\rho$ and $\eta$ with $\rho_t=\frac{\rho_0}{\sqrt{t}}$ and $\eta_t = \frac{\eta_0}{\sqrt{t}}$, we have
\begin{align*}
   \E f(\w_{t+1}) \leq& \E f(\w_t) - \frac{\eta_t}{2} \E ||\nabla f(\w_t)||^2 + L\eta_t^2 \sigma^2 + 2L^2\eta_t\rho_t^2 - L^3\eta_t^2\rho_t^2 + (1-L\eta_t)\eta_t L G \rho_t.
   \\
   \leq& \E f(\w_t) - \frac{\eta_t}{2} \E ||\nabla f(\w_t)||^2 + L\eta_t^2 \sigma^2 + 2L^2\eta_t\rho_t^2 + (1-L\eta_t)\eta_t L G \rho_t.
\end{align*}
Take telescope sum, we have
\begin{align*}
    \E f(\w_{T}) - f(\w_0) \leq&  - \sum_{t=1}^T \frac{\eta_t}{2} \E ||\nabla f(\w_t)||^2 + (L \sigma^2 \eta_0^2 + LG\rho_0\eta_0) \sum_{t=1}^T\frac{1}{t} + (2L^2\eta_0\rho_0^2 - L^2G\eta_0^2\rho_0)\sum_{t=1}^T\frac{1}{t^{\frac{3}{2}}} 
\end{align*}
Under $\rho_0 \leq \frac{1}{2}G \eta_0$, the last term will be less than 0, which means:
\begin{align*}
    \E f(\w_{T}) - f(\w_0) \leq&  - \sum_{t=1}^T \frac{\eta_t}{2} \E ||\nabla f(\w_t)||^2 + (L \sigma^2 \eta_0^2 + LG\rho_0\eta_0) \sum_{t=1}^T\frac{1}{t}.
\end{align*}
With 
\begin{align*}
    \frac{\eta_T}{2}\sum_{t=1}^T \E||\nabla f(\w_t)||^2 \leq& \sum_{t=1}^T \frac{\eta_t}{2} \E ||\nabla f(\w_t)||^2 \leq f(\w_0) - \E f(\w_{T}) + (L \sigma^2 \eta_0^2 + LG\rho_0\eta_0) \sum_{t=1}^T\frac{1}{t},
\end{align*}
we have 
\begin{align*}
    \frac{\eta_0}{2\sqrt{T}}\sum_{t=1}^T \E||\nabla f(\w_t)||^2 \leq& f(\w_0) - \E f(\w_{T}) + (L \sigma^2 \eta_0^2 + LG\rho_0\eta_0) \sum_{t=1}^T\frac{1}{t}
    \\
    \leq& f(\w_0) - \E f(\w_{T}) + (L \sigma^2 \eta_0^2 + LG\rho_0\eta_0) \log T.
\end{align*}
Finally, we achieve the result:
\begin{align*}
    \frac{1}{T}\sum_{t=1}^T\E||\nabla f(\w_t)||^2 \leq& \frac{2 \cdot (f(\w_0) - \E f(\w_{T}))}{\eta_0}\frac{1}{\sqrt{T}} + 2 (L \sigma^2 \eta_0 + LG\rho_0)\frac{\log T}{\sqrt{T}},
\end{align*}
which shows that SAM can converge at the rate of $O(\log T/\sqrt{T})$.
\end{proof}
\end{proposition}

\subsection{Proof of Theorem 2}
Suppose we can obtain the noisy observation gradient $g(\w_t)$ of true gradient $\nabla f(\w_t)$, and the mask $\m$, we can write the iteration of SAM.
Consider the iteration of Sparse SAM:
\begin{align}
\label{equ:iter-ssam}
\left\{
    \begin{array}{ll}
        \tilde{\w}_{t+\frac{1}{2}}=&\w_t + \rho \frac{g(\w_t)}{||g(\w_t)||} \odot \m_t \\
        \w_{t+1}=&\w_t - g(\tilde{\w}_{t+\frac{1}{2}})
    \end{array}
\right.
\end{align}
Let us denote the difference as $\w_{t+\frac{1}{2}} - \tilde{\w}_{t+\frac{1}{2}}=\boldsymbol{e}_t$.

\begin{lemma}
\label{lemma:4}
With $\rho > 0$, we have:
\begin{align*}
    \E\langle \nabla f(\w_t), g(\tilde{\w}_{t + \frac{1}{2}})\rangle 
    \geq
    \frac{1}{2}||\nabla f(\w_t)||^2 - 4 L^2 \rho^2 - L \rho G - L^2 ||\boldsymbol{e}_t||^2
\end{align*}

\begin{proof}
Similar to Lemma~\ref{lemma:2}, We denote the true gradient as $\hat{\w}_{t+\frac{1}{2}} = \w_t + \rho \frac{\nabla f(\w_t)}{||\nabla f(\w_t)||}$, and also add and subtract the item $g(\tilde{\w}_{t + \frac{1}{2}})$:
\begin{align*}
    \langle \nabla f(\w_t), g(\tilde{\w}_{t + \frac{1}{2}})\rangle 
    =&
    \langle \nabla f(\w_t), g(\tilde{\w}_{t + \frac{1}{2}}) - g(\hat{\w}_{t+\frac{1}{2}})\rangle  + \langle g(\hat{\w}_{t+\frac{1}{2}}), \nabla f(\w_t) \rangle
\end{align*}
For the first term, we bound it by using the smoothness of $g(\w)$:
\begin{align*}
    - \langle \nabla f(\w_t), g(\tilde{\w}_{t + \frac{1}{2}}) - g(\hat{\w}_{t+\frac{1}{2}})\rangle
    \leq& \frac{1}{2}||g(\tilde{\w}_{t + \frac{1}{2}}) - g(\hat{\w}_{t+\frac{1}{2}})||^2 + \frac{1}{2}||\nabla f(\w_t)||^2
    \\
    \leq& \frac{L^2}{2} ||\tilde{\w}_{t + \frac{1}{2}} - \hat{\w}_{t+\frac{1}{2}}||^2 + \frac{1}{2}||\nabla f(\w_t)||^2
    \\
    =& \frac{L^2}{2} ||\w_{t + \frac{1}{2}} - \boldsymbol{e}_t - \hat{\w}_{t+\frac{1}{2}}||^2 + \frac{1}{2}||\nabla f(\w_t)||^2
    \\
    \leq& L^2 (||\w_{t+\frac{1}{2}} - \hat{\w}_{t+\frac{1}{2}}||^2 + ||\boldsymbol{e}_t||^2) + \frac{1}{2}||\nabla f(\w_t)||^2
    \\
    =& L^2 (\rho^2 ||\frac{g(\w_t)}{||g(\w_t)||} - \frac{\nabla f(\w_t)}{||\nabla f(\w_t)||}||^2 + ||\boldsymbol{e}_t||^2) + \frac{1}{2}||\nabla f(\w_t)||^2
    \\
    \leq& 4L^2 \rho^2 + L^2||\boldsymbol{e}_t||^2 + \frac{1}{2}||\nabla f(\w_t)||^2
\end{align*}
For the second term, we do the same in Lemma~\ref{lemma:2}:
\begin{align*}
    \E\langle g(\hat{\w}_{t+\frac{1}{2}}), \nabla f(\w_t) \rangle
    % =&
    % \langle \nabla f(\hat{w}_{t+\frac{1}{2}}), \nabla f(w_t) \rangle
    % \\
    % =& \langle \nabla f(w_t + \rho \frac{\nabla f(w_t)}{||\nabla f(w_t)||^2}), \nabla f(w_t) \rangle
    % \\
    \geq& ||\nabla f(\w_t)||^2 - \rho L G.
\end{align*}
Assembling the two inequalities yields to the result. 
\end{proof}
\end{lemma}

%%%%%%%%%%%%%%%% lemma 5 %%%%%%%%%%%%%%%%%%%%%%%%%
\begin{lemma}
For $\eta \leq \frac{1}{L}$, the iteration~\ref{equ:iter-ssam} satisfies for all $t>0$:
\begin{align*}
    \E f(\w_{t+1}) \leq& \E f(\w_t) - \frac{\eta}{2} \E ||\nabla f(\w_t)||^2 + L\eta^2\sigma^2 + 2\eta L^2\rho^2(2-L\eta) + (1-L\eta)\eta LG\rho \\
    & + (1+L\eta)\eta L^2||\boldsymbol{e}_t||^2
\end{align*}
\end{lemma}

\begin{proof}
By the smoothness of the function $f$, we obtain
\begin{align*}
    f(\w_{t+1}) 
    \leq& 
    f(\w_t) - \eta \langle \nabla f(\w_t), g(\tilde{\w}_{t+\frac{1}{2}}) \rangle + \frac{L\eta^2}{2}||g(\tilde{\w}_{t+\frac{1}{2}})||^2
    \\
    =& f(\w_t) - \eta \langle \nabla f(\w_t), g(\tilde{\w}_{t+\frac{1}{2}}) \rangle + \frac{L\eta^2}{2} ( ||\nabla f(\w_t)-g(\tilde{\w}_{t+\frac{1}{2}})||^2-||\nabla f(\w_t)||^2 + 2\langle \nabla f(\w_t), g(\tilde{\w}_{t+\frac{1}{2}})\rangle)
    \\
    =& f(\w_t) -\frac{L\eta^2}{2}||\nabla f(\w_t)||^2 + \frac{L\eta^2}{2} ||\nabla f(\w_t) - g(\tilde{\w}_{t+\frac{1}{2}})||^2 - (1-L\eta)\eta \langle  \nabla f(\w_t), g(\tilde{\w}_{t+\frac{1}{2}})\rangle
    \\
    \leq& f(\w_t) - \frac{L\eta^2}{2}||\nabla f(\w_t)||^2 + L\eta^2 ||\nabla f(\w_t) - g(\w_t)||^2 + L\eta^2 ||g(\w_t) - g(\tilde{\w}_{t+\frac{1}{2}})||^2 \\& - (1-L\eta)\eta \langle  \nabla f(\w_t), g(\tilde{\w}_{t+\frac{1}{2}})\rangle
    \\
    \leq& f(\w_t) - \frac{L\eta^2}{2}||\nabla f(\w_t)||^2 + L\eta^2 ||\nabla f(\w_t) - g(\w_t)||^2 + L\eta^2 L^2||\w_t - \tilde{\w}_{t+\frac{1}{2}}||^2 \\& - (1-L\eta)\eta \langle  \nabla f(\w_t), g(\tilde{\w}_{t+\frac{1}{2}})\rangle
    \\
    % =& f(w_t) - \frac{L\eta^2}{2}||\nabla f(w_t)||^2 + L\eta^2 ||\nabla f(w_t) - g(w_t)||^2 + \eta^2L^3 ||\rho \cdot \frac{g(w_t)}{||g(w_t)||} \odot m_t||^2
    =& f(\w_t) - \frac{L\eta^2}{2}||\nabla f(\w_t)||^2 + L\eta^2 ||\nabla f(\w_t) - g(\w_t)||^2 + \eta^2L^3 ||\w_t - \w_{t+\frac{1}{2}} + \boldsymbol{e}_t||^2
    \\& -  (1-L\eta)\eta \langle  \nabla f(\w_t), g(\tilde{\w}_{t+\frac{1}{2}} )\rangle
    \\
    \leq& f(\w_t) - \frac{L\eta^2}{2}||\nabla f(\w_t)||^2 + L\eta^2 ||\nabla f(\w_t) - g(\w_t)||^2 + 2 \eta^2L^3 (||\w_t - \w_{t+\frac{1}{2}}||^2 + ||\boldsymbol{e}_t||^2)
    \\& -  (1-L\eta)\eta \langle  \nabla f(\w_t), g(\tilde{\w}_{t+\frac{1}{2}} )\rangle
    \\
    =& f(\w_t) - \frac{L\eta^2}{2}||\nabla f(\w_t)||^2 + L\eta^2 ||\nabla f(\w_t) - g(\w_t)||^2 + 2 \eta^2L^3 ( \rho^2 + ||\boldsymbol{e}_t||^2)
    \\& -  (1-L\eta)\eta \langle  \nabla f(\w_t), g(\tilde{\w}_{t+\frac{1}{2}} )\rangle
\end{align*}
Taking the expectation and using Lemma~\ref{lemma:4} we obtain
\begin{align*}
\E f(\w_{t+1}) 
\leq& 
\E f(\w_t) - \frac{L\eta^2}{2}\E ||\nabla f(\w_t)||^2 + L\eta^2\E ||\nabla f(\w_t) - g(\w_t)||^2 + 2\eta^2L^3 ( \rho^2 + ||\boldsymbol{e}_t||^2) \\& - (1-L\eta)\eta \E \langle  \nabla f(\w_t), g(\tilde{\w}_{t+\frac{1}{2}})\rangle
\\
\leq& \E f(\w_t) - \frac{L\eta^2}{2}\E ||\nabla f(\w_t)||^2 + L\eta^2 \sigma^2 + 2\eta^2L^3 ( \rho^2 + ||\boldsymbol{e}_t||^2) \\& - (1-L\eta)\eta \E \langle  \nabla f(\w_t), g(\tilde{\w}_{t+\frac{1}{2}})\rangle
\\
\leq& \E f(\w_t) - \frac{L\eta^2}{2}\E ||\nabla f(\w_t)||^2 + L\eta^2 \sigma^2 + 2\eta^2L^3 ( \rho^2 + ||\boldsymbol{e}_t||^2) \\& - (1-L\eta)\eta \left[\frac{1}{2}||\nabla f(\w_t)||^2 - 4 L^2 \rho^2 - L \rho G - L^2 ||\boldsymbol{e}_t||^2\right]
\\
=& \E f(\w_t) - \frac{\eta}{2} \E ||\nabla f(\w_t)||^2 + L\eta^2\sigma^2 + 2\eta L^2\rho^2(2-L\eta) + (1-L\eta)\eta LG\rho \\& + (1+L\eta)\eta L^2||\rho \frac{g(\w_t)}{||g(\w_t)||} \odot \m_t - \rho \frac{g(\w_t)}{||g(\w_t)||}||^2
\\
=& \E f(\w_t) - \frac{\eta}{2} \E ||\nabla f(\w_t)||^2 + L\eta^2\sigma^2 + 2\eta L^2\rho^2(2-L\eta) + (1-L\eta)\eta LG\rho \\& + (1+L\eta)\eta L^2||\boldsymbol{e}_t||^2
\end{align*}
\end{proof}

\begin{proposition}
\label[Proposition]{pro:2}
Let us $\eta_t=\frac{\eta_0}{\sqrt{t}}$ and perturbation amplitude $\rho$ decay with square root of $t$, \emph{e.g.}, $\rho_t=\frac{\rho_0}{\sqrt{t}}$. With $\rho_0 \leq G\eta_0 / 5$, we have:
\begin{align*}
    \frac{1}{T}\sum_{t=1}^T \E||\nabla f(\w)||^2 \leq&
    C_3\frac{1}{\sqrt{T}} + C_4\frac{\log T}{\sqrt{T}},
\end{align*}
where $C_3=\frac{2}{\eta_0}(f(\w_0-\E f(\w_T)+L^3\eta_0^2\rho_0^2\frac{\pi^2}{6})$ and $C_4=2(L\sigma^2\eta_0+LG\rho_0)$.
\end{proposition}

\begin{proof}
By taking the expectation and using Lemma~\ref{lemma:4}, and taking the schedule to be $\eta_t = \frac{\eta_0}{\sqrt{t}}$, $\rho_t = \frac{\rho _0}{\sqrt{t}}$, we obtain:
\begin{align*}
    \E f(\w_{t+1}) \leq& \E f(\w_t) - \frac{\eta_t}{2} \E ||\nabla f(\w_t)||^2 + L\eta_t^2\sigma^2 + 2\eta_t L^2\rho_t^2(2-L\eta_t) + (1-L\eta_t)\eta_t LG\rho_t \\& + (1+L\eta_t)\eta_t L^2||\boldsymbol{e}_t||^2
    \\ 
    \leq&  \E f(\w_t) - \frac{\eta_t}{2} \E ||\nabla f(\w_t)||^2 + L\eta_t^2\sigma^2 + 2\eta_t L^2\rho_t^2(2-L\eta_t) + (1-L\eta_t)\eta_t LG\rho_t \\& + (1+L\eta_t)\eta_t L^2\rho_t^2
\end{align*}
By taking sum, we have
\begin{align*}
    \E f(\w_T)-f(\w_0) \leq& -\sum_{t=1}^{T}\frac{\eta_t}{2}\E||\nabla f(\w_t)||^2 + (L\eta_0^2\sigma^2 + LG\eta_0\rho_0)\sum_{t=1}^{T}\frac{1}{t} \\&+ L^2\rho_0\eta_0(4\rho_0 - G\eta_0)\sum_{t=1}^{T}\frac{1}{t^{\frac{3}{2}}}     \\&+ \sum_{t=1}^T(1+L\eta_t)\eta_t L^2\rho_t^2
    \\
    \leq& -\sum_{t=1}^{T}\frac{\eta_t}{2}\E||\nabla f(\w_t)||^2 + (L\eta_0^2\sigma^2 + LG\eta_0\rho_0)\sum_{t=1}^{T}\frac{1}{t} \\&+ L^2\rho_0\eta_0(4\rho_0 - G\eta_0)\sum_{t=1}^{T}\frac{1}{t^{\frac{3}{2}}}  \\&+\eta_0L^2\rho_0^2 \sum_{t=1}^{T}\frac{1}{t^{\frac{3}{2}}} + L^3\eta_0^2\rho_0^2 \sum_{t=1}^{T}\frac{1}{t^2}
\end{align*}
and finally, we have,
\begin{align*}
    \E f(\w_T)-f(\w_0) \leq& -\sum_{t=1}^{T}\frac{\eta_t}{2}\E||\nabla f(\w_t)||^2 + (L\eta_0^2\sigma^2 + LG\eta_0\rho_0)\sum_{t=1}^{T}\frac{1}{t} \\&+L^2\rho_0\eta_0(5\rho_0 - G\eta_0)\sum_{t=1}^{T}\frac{1}{t^{\frac{3}{2}}} + L^3\eta_0^2\rho_0^2 \frac{\pi^2}{6}
\end{align*}

Bound $\rho_0$ with $\frac{G\eta_0}{5}$, we have:
\begin{align*}
    \frac{\eta_0}{2\sqrt{T}}\sum_{t=1}^T \E||\nabla f(\w_t)||^2 \leq& f(\w_0) - \E f(\w_{T}) + (L \sigma^2 \eta_0^2 + LG\rho_0\eta_0) \sum_{t=1}^T\frac{1}{t} \\ & + L^3\eta_0^2\rho_0^2 \sum_{t=1}^{T}\frac{1}{t^2}
\end{align*}
Finally, we achieve the result:
\begin{align*}
    \frac{1}{T}\sum_{t=1}^T\E ||\nabla f(\w_t)||^2\leq& \frac{2\left(f(\w_0)-\E f(\w_T)+L^3\eta_0^2\rho_0^2 \frac{\pi^2}{6}\right)}{\eta_0}\frac{1}{\sqrt{T}} \\
    & + 2(L\sigma^2\eta_0+LG\rho_0)\frac{\log T}{\sqrt{T}}
\end{align*}
\end{proof}
So far, we have completed the proof of the theory in the main text.

%%%%%%%%%%%%%%%%%%%%%%%%%%%%%%%%%%%%%%%%%%%%%%%%%%%%%%%%%%%%%%%%%%%%%%%%%% Generalization Error Bound %%%%%%%%%%%%%%%%%%%%%%%%%%%%%%%%%%
\subsection{Proof of Theorem 3}
The proof is mainly based on the generalization error in Sharpness-Aware Minimization~\cite{sam}. Similar to~\cite{sam}, we have the condition $L_\mathscr{D}(\boldsymbol{w}) \leq \mathbb{E}_{\boldsymbol{\epsilon} \sim \mathcal{N}(\boldsymbol{0},\boldsymbol{\rho})}[L_\mathscr{D}(\boldsymbol{w}+\boldsymbol{m}\odot \boldsymbol{\epsilon})]$ means that adding Gaussian perturbation in a specific dimension,~\emph{i.e.}, the element of $\boldsymbol{m}$ is 1, should not decrease the test error, which is more loose compared to SAM's. This condition is satisfied only for the final solution of convergence.

\begin{proof}
There is a lower bounded in the right-hand of the theorem, $\sqrt{k\log(1+\|\boldsymbol{w}\|_2^2/\rho^2)/(4n)}$. This could be achieved if $\|\boldsymbol{w}\|_2^2 > \rho^2(\exp(4n/k)-1)$ trivially. Therefore, we only consider the case when $\|\boldsymbol{w}\|_2^2 \leq \rho^2(\exp(4n/k)-1)$ in the rest of the proof.

Inspired from \cite{chatterji2019intriguing}, using PAC-Bayesian generalization bound~\cite{pac_bayesian} and following \cite{dziugaite2017computing}, the following generalization bound holds for any prior $\mathscr{P}$ over parameters with probability $1-\delta$ over the choice of the training set $\mathcal{S}$, for any posterior $\mathscr{Q}$ over parameters:
\begin{equation}
    \mathbb{E}_{\boldsymbol{w}\sim \mathscr{Q}}[L_\mathscr{D}(\boldsymbol{w})] \leq \mathbb{E}_{\boldsymbol{w}\sim \mathscr{Q}}[L_\mathcal{S}(\boldsymbol{w})] +\sqrt{\frac{KL(\mathscr{Q}||\mathscr{P}) + \log \frac{n}{\delta}}{2(n-1)}}
\end{equation}
Moreover, if $\mathscr{P}=\mathcal{N}(\boldsymbol{\mu}_P,\sigma_P^2 \boldsymbol{I})$ and $\mathscr{Q}=\mathcal{N}(\boldsymbol{\mu}_Q,\sigma_Q^2 \boldsymbol{I})$, then the KL divergence can be written as follows:
\begin{equation}
    KL(\mathscr{P}||\mathscr{Q}) = \frac{1}{2}\bigg[\frac{k\sigma_Q^2+\|\boldsymbol{\mu}_P-\boldsymbol{\mu}_Q\|_2^2}{\sigma_P^2}- k + k\log\left(\frac{\sigma_P^2}{\sigma_Q^2}\right)\bigg]
\end{equation}
Given a posterior standard deviation $\sigma_Q$, one could choose a prior standard deviation $\sigma_P$ to minimize the above KL divergence and hence the generalization bound by taking the derivative\footnote{Despite the nonconvexity of the function here in $\sigma_P^2$, it has a unique stationary point which happens to be its minimizer.} of the above KL with respect to $\sigma_P$ and setting it to zero. We would then have ${\sigma^*_P}^2=\sigma_Q^2+\|\boldsymbol{\mu}_P-\boldsymbol{\mu}_Q\|_2^2/k$. However, since $\sigma_P$ should be chosen before observing the training data $\mathcal{S}$ and $\boldsymbol{\mu}_Q$,$\sigma_Q$ could depend on $\mathcal{S}$, we are not allowed to optimize $\sigma_P$ in this way. Instead, one can have a set of predefined values for $\sigma_P$ and pick the best one in that set. Given fixed $a,b>0$, let $T=\{c\exp((1-j)/k)| j\in \mathbb{N}\}$ be that predefined set of values for $\sigma_P^2$. If for any $j\in \mathbb{N}$, the above PAC-Bayesian bound holds for $\sigma_P^2=c\exp((1-j)/k)$ with probability $1-\delta_j$ with $\delta_j=\frac{6\delta}{\pi^2j^2}$, then by the union bound, all above bounds hold simultaneously with probability at least $1-\sum_{j=1}^\infty\frac{6\delta}{\pi^2j^2} =1-\delta$.

Let $\sigma_Q=\rho$, $\boldsymbol{\mu}_Q=\boldsymbol{w}$ and $\boldsymbol{\mu}_P=\boldsymbol{0}$. Therefore, we have:
\begin{equation}\label{eq:weight_upper}
    \sigma_Q^2+\|\boldsymbol{\mu}_P-\boldsymbol{\mu}_Q\|_2^2/k \leq \rho^2 + \|\boldsymbol{w}\|_2^2/k \leq \rho^2(1+\exp(4n/k))
\end{equation}
We now consider the bound that corresponds to $j=\lfloor 1 - k\log((\rho^2 + \|\boldsymbol{w}\|_2^2/k)/c)\rfloor$. We can ensure that $j\in \mathbb{N}$ using inequality \eqref{eq:weight_upper} and by setting $c=\rho^2(1+\exp(4n/k))$. Furthermore, for $\sigma_P^2=c\exp((1-j)/k)$, we have:
\begin{equation}
    \rho^2 + \|\boldsymbol{w}\|_2^2/k \leq \sigma_P^2 \leq \exp(1/k) \left(\rho^2 + \|\boldsymbol{w}\|_2^2/k\right)
\end{equation}

Therefore, using the above value for $\sigma_P$, KL divergence can be bounded as follows:

\begin{align}
    KL(\mathscr{P}||\mathscr{Q}) &=\frac{1}{2}\bigg[\frac{k\sigma_Q^2+\|\boldsymbol{\mu}_P-\boldsymbol{\mu}_Q\|_2^2}{\sigma_P^2}- k + k\log\left(\frac{\sigma_P^2}{\sigma_Q^2}\right)\bigg]\\
    &\leq \frac{1}{2}\bigg[\frac{k(\rho^2 + \|\boldsymbol{w}\|_2^2/k)}{\rho^2 + \|\boldsymbol{w}\|_2^2/k}- k + k\log\left(\frac{\exp(1/k) \left(\rho^2 + \|\boldsymbol{w}\|_2^2/k\right)}{\rho^2}\right)\bigg]\\
    &=\frac{1}{2}\bigg[k\log\left(\frac{\exp(1/k) \left(\rho^2 + \|\boldsymbol{w}\|_2^2/k\right)}{\rho^2}\right)\bigg]\\    
    &= \frac{1}{2}\bigg[1+k\log\left(1+\frac{\|\boldsymbol{w}\|_2^2}{k\sigma_Q^2}\right)\bigg]
\end{align}
Given the bound that corresponds to $j$ holds with probability $1-\delta_j$ for $\delta_j=\frac{6\delta}{\pi^2j^2}$, the log term in the bound can be written as:
\begin{align*}
    \log\frac{n}{\delta_j} &= \log\frac{n}{\delta} + \log\frac{\pi^2j^2}{6}\\
    &\leq \log\frac{n}{\delta} + \log\frac{\pi^2k^2\log^2(c/(\rho^2 + \|\boldsymbol{w}\|_2^2/k))}{6}\\
    &\leq \log\frac{n}{\delta} + \log\frac{\pi^2k^2\log^2(c/\rho^2)}{6}\\
    &\leq \log\frac{n}{\delta} + \log\frac{\pi^2k^2\log^2(1+\exp(4n/k))}{6}\\
    &\leq \log\frac{n}{\delta} + \log\frac{\pi^2k^2(2 +4n/k)^2}{6}\\
    &\leq \log \frac{n}{\delta} + 2 \log\left(6n + 3 k\right)
\end{align*}
Therefore, the generalization bound can be written as follows:
\begin{equation}\label{eq:almost_theorem}
    \mathbb{E}_{\epsilon_i \sim \mathcal{N}(0,\sigma)}[L_\mathscr{D}(\boldsymbol{w}+\boldsymbol{\epsilon})] \leq \mathbb{E}_{\epsilon_i \sim \mathcal{N}(0,\sigma)}[L_\mathcal{S}(\boldsymbol{w}+\boldsymbol{\epsilon})] +\sqrt{\frac{\frac{1}{4}k\log\left(1+\frac{\|\boldsymbol{w}\|_2^2}{k\sigma^2}\right) + \frac{1}{4}+\log \frac{n}{\delta} + 2\log\left(6n + 3k\right)}{n-1}}
\end{equation}
In the above bound, we have $\epsilon_i \sim \mathcal{N}(0,\sigma)$. Therefore, $\|\boldsymbol{\epsilon}\|_2^2$ has chi-square distribution and by Lemma 1 in \cite{laurent2000adaptive}, we have that for any positive $t$:
\begin{equation}
    P(\|\boldsymbol{\epsilon}\|_2^2 - k\sigma^2 \geq 2\sigma^2\sqrt{kt} + 2t\sigma^2) \leq \exp(-t)
\end{equation}
Therefore, with probability $1-1/\sqrt{n}$ we have that:
$$||m\odot \epsilon||^2  \leq \|\boldsymbol{\epsilon}\|_2^2\leq \sigma^2(2\ln(\sqrt{n})+k + 2\sqrt{k\ln(\sqrt{n})}) \leq \sigma^2k\left(1+\sqrt{\frac{\ln(n)}{k}}\right)^2\leq \rho^2$$
Substituting the above value for $\sigma$ back to the inequality and using theorem's assumption gives us following inequality:
\begin{align*}
    L_\mathscr{D}(\boldsymbol{w}) &\leq (1-1/\sqrt{n})\max_{\|\boldsymbol{m\odot\epsilon}\|_2 \leq \rho} L_\mathcal{S}(\boldsymbol{w} + \boldsymbol{\epsilon}) + 1/\sqrt{n}\\
    &+\sqrt{\frac{\frac{1}{4}k\log\left(1+\frac{\|\boldsymbol{w}\|_2^2}{\rho^2}\left(1+\sqrt{\frac{\log(n)}{k}}\right)^2\right) + \log\frac{n}{\delta} + 2\log\left(6n + 3k\right)}{n-1}}\\
    &\leq \max_{\|\boldsymbol{\epsilon}\|_2 \leq \rho} L_\mathcal{S}(\boldsymbol{w} + \boldsymbol{m\odot\epsilon}) + \\
    &+\sqrt{\frac{k\log\left(1+\frac{\|\boldsymbol{w}\|_2^2}{\rho^2}\left(1+\sqrt{\frac{\log(n)}{k}}\right)^2\right) + 4\log\frac{n}{\delta} + 8\log\left(6n + 3k\right)}{n-1}}
\end{align*}

\end{proof}

\section{More Experimets}

\textbf{SAM with different perturbation magnitude $\rho$.}
We determine the perturbation magnitude $\rho$ by using grid search. We choose $\rho$ from the set $\{0.01, 0.02, 0.05, 0.1, 0.2, 0.5\}$ for CIFAR, and choose $\rho$ from $\{0.01, 0.02, 0.05, 0.07, 0.1, 0.2\}$ for ImageNet. We show the results when varying $\rho$ in~\cref{table:rho-cifar} and~\cref{table:rho-imagenet}. From this table, we can see that the $\rho=0.1$, $\rho=0.2$ and $\rho=0.07$ is sutiable for CIFAR10, CIFAR100 and ImageNet respectively.

\begin{table}[ht]
\centering
\vspace{-2mm}
\caption{Test accuracy of ResNet18 and WideResNet28-10 on CIFAR10 and CIFAR100 with different perturbation magnitude $\rho$.}
\label[table]{table:rho-cifar}
% \resizebox{1\textwidth}{12mm}{
\begin{tabular}{cccccccc}
\toprule
Dataset & SAM $\rho$ & 0.01    & 0.02    & 0.05    & 0.1              & 0.2              & 0.5     \\ \hline
         & ResNet18   & 96.58\% & 96.54\% & 96.68\% & \textbf{96.83\%} & 96.32\%          & 93.16\% \\
\multirow{-2}{*}{CIFAR10}  & WideResNet28-10 & 97.26\% & 97.34\% & 97.31\% & \textbf{97.48\%} & 97.29\%          & 95.13\% \\ \hline
         & ResNet18   & 79.56\% & 79.98\% & 80.71\% & 80.65\% & \textbf{81.03\%} & 77.57\% \\
\multirow{-2}{*}{CIFAR100} & WideResNet28-10 & 82.25\% & 83.04\% & 83.47\% & 83.47\% & \textbf{84.20\%} & 84.03\% \\ \bottomrule
\end{tabular}
% }
\end{table}

\begin{table}[ht]
\centering
\vspace{-3mm}
\caption{Test accuracy of ResNet50 on ImageNet with different perturbation magnitude $\rho$.}

\label{table:rho-imagenet}
\begin{tabular}{cccccccc}
\toprule
datasets & SAM $\rho$ & 0.01    & 0.02    & 0.05    & 0.07             & 0.1     & 0.2     \\ \hline
ImageNet & ResNet50   & 76.63\% & 76.78\% & 77.12\% & \textbf{77.25\%} & 77.00\% & 76.37\% \\ \bottomrule
\end{tabular}
\end{table}

\textbf{Influence of hyper-parameters.}
We first examine the effect of the number of sample size $N_F$ of SSAM-F in~\cref{ablation-num-samples}. From it we can see that a certain number of samples is enough for the approximation of data distribution in SSAM-F,~\emph{e.g.}, $N_F=128$, which greatly saves the computational cost of SSAM-F. In~\cref{ablation-update-mask}, we also report the influence of the mask update interval on SSAM-F and SSAM-D. The results show that the performance degrades as the interval becom longer, suggesting that dense mask updates are necessary for our methods. Both of them are ResNet18 on CIFAR10.

\begin{table}[ht]
	\begin{minipage}[t]{0.53\textwidth}
	\small
		\centering
		\caption{Results of ResNet18 on CIFAR10 with different number of samples $N_F$ in SSAM-F. `Time' reported in table is the time cost to calculate Fisher Information based on $N_F$ samples.}
		\label{ablation-num-samples}
		\begin{tabular}{cccc}
        \toprule
        Sparsity & $N_F$ & Acc & Time \\ \hline
        \multirow{6}{*}{0.5} & 16 & 96.77\% & 1.49s \\ 
         & 128 & 96.84\% & 4.40s \\ 
         & 512 & 96.67\% & 15.35s \\ 
         & 1024 & 96.83\% & 30.99s \\ 
         & 2048 & 96.68\% & 56.23s \\ 
         & 4096 & 96.66\% & 109.31s \\ \hline
         \multirow{6}{*}{0.9} & 16 &  96.79\% & 1.47s\\
         & 128 & 96.50\% & 5.42s\\
         & 512 & 96.43\% & 15.57s\\
         & 1024 & 96.75\% & 29.24s \\
         & 2048 & 96.62\% & 57.72s\\
         & 4096 & 96.59\% & 110.65s\\ \bottomrule
        \end{tabular}
	\end{minipage}
	\hspace{+3mm}
	\begin{minipage}[t]{0.43\textwidth}
	\small
	\caption{Results of ResNet18 on CIFAR10 with different $T_m$ intervals of update mask. The left of `/' is accuracy of SSAM-F, while the right is SSAM-D.}
		\label{ablation-update-mask}
		\centering
		\begin{tabular}{ccc}
        \toprule
        Sparsity & $T_m$ & Acc \\ \hline
        \multirow{6}{*}{0.5} & 1 & 96.81\%/96.74\% \\
         & 2 & 96.51\%/96.74\% \\
         & 5 & 96.83\%/96.60\% \\
         & 10 & 96.71\%/96.73\% \\
         & 50 & 96.65\%/96.75\% \\
         & Fixed & 96.57\%/96.52\% \\ \hline
        \multirow{6}{*}{0.9} & 1 & 96.70\%/96.65\% \\
         & 2 & 93.75\%/96.63\% \\
         & 5 & 96.51\%/96.69\% \\
         & 10 & 96.67\%/96.74\% \\
         & 50 & 96.64\%/96.66\% \\
         & Fixed & 96.21\%/96.46\% \\ \bottomrule
        \end{tabular}
	\end{minipage}
\end{table}

\textbf{Details of ViT experiments in Sec.~\ref{sec:n_mstructure_experiment}.} In order to achieve true training acceleration, the experiments with ViT in Sec.~\ref{sec:n_mstructure_experiment} utilize the \textit{cusparseLt} library~\cite{cusparselt} developed by Nvidia, which is a high-performance CUDA library dedicated to general matrix-matrix operations in which at least one operand is a sparse matrix. Based on official test results in~\cite{bert_with_spmm} and our own testing experience, it can be observed that the larger the dimension of the matrix, the more significant the acceleration effect of sparse matrix multiplication implemented by the \textit{cusparseLt} library,~\emph{e.g.}, \textit{cusparseLt} can accelerate sparse matrix multiplication by 1.6 times in the case of two matrices with dimensions 4096, 10684, and 1024.

The structure of ViT follows the original design of~\cite{vit}, except that we set the dimensions as large as possible in order to highlight the acceleration effect of \textit{cusparseLt}, as shown in the Table~\ref{appendix_vit_model_setting}. Considering the reality that we are limited by hardware, the dimension of the transformer we modified the training settings as shown in Table~\ref{appendix_vit_train_setting}.

It should be emphasized that the experimental purpose of Sec.~\ref{sec:n_mstructure_experiment} is to demonstrate that our proposed SSAM can achieve real training acceleration through CUDA libraries for relevant sparse matrices. The figure of running time~\emph{v.s.} training iteration is shown in figure~\ref{speedup}. Our code using~\textit{cusparseLt} still has the optimization possibility to further improve the computational speed.

\begin{figure}[ht]
  \centering
  \begin{minipage}[b]{0.3\linewidth}
    \centering
    {
        \captionof{table}{Setting of ViT structure}
        \label{appendix_vit_model_setting}
        \begin{tabular}{cc}
        \toprule
        \multicolumn{2}{c}{ViT structure} \\ \hline
        image size & $32\times32$ \\
        patch size & $1\times1$ \\
        dim &  8192\\
        mlp ratio & 1\\
        heads &  1\\
        number of blocks &  1\\
        dropout &  0\\
        embedding dropout &  0\\ \bottomrule
        \end{tabular}%
    }
  \end{minipage}
  \begin{minipage}[b]{0.3\linewidth}
    \centering
    {
        \captionof{table}{Hyper-parameters}
        \label{appendix_vit_train_setting}
        \begin{tabular}{cc}
        \toprule
        \multicolumn{2}{c}{Training setting} \\ \hline
        epochs & 10 \\
        batch size & 16 \\
        basic optimizer & Adam \\
        learning rate & 1e-3 \\
        weight decay & 0 \\
        beta & (0.9, 0.999) \\
        rho & 0.2\\ 
        number of samples & 512\\ \bottomrule
        \end{tabular}%
    }
    \end{minipage}
\end{figure}

\begin{figure}[ht]
    \centering
    \vspace{-0.2cm}
    \includegraphics[width=0.5\textwidth]{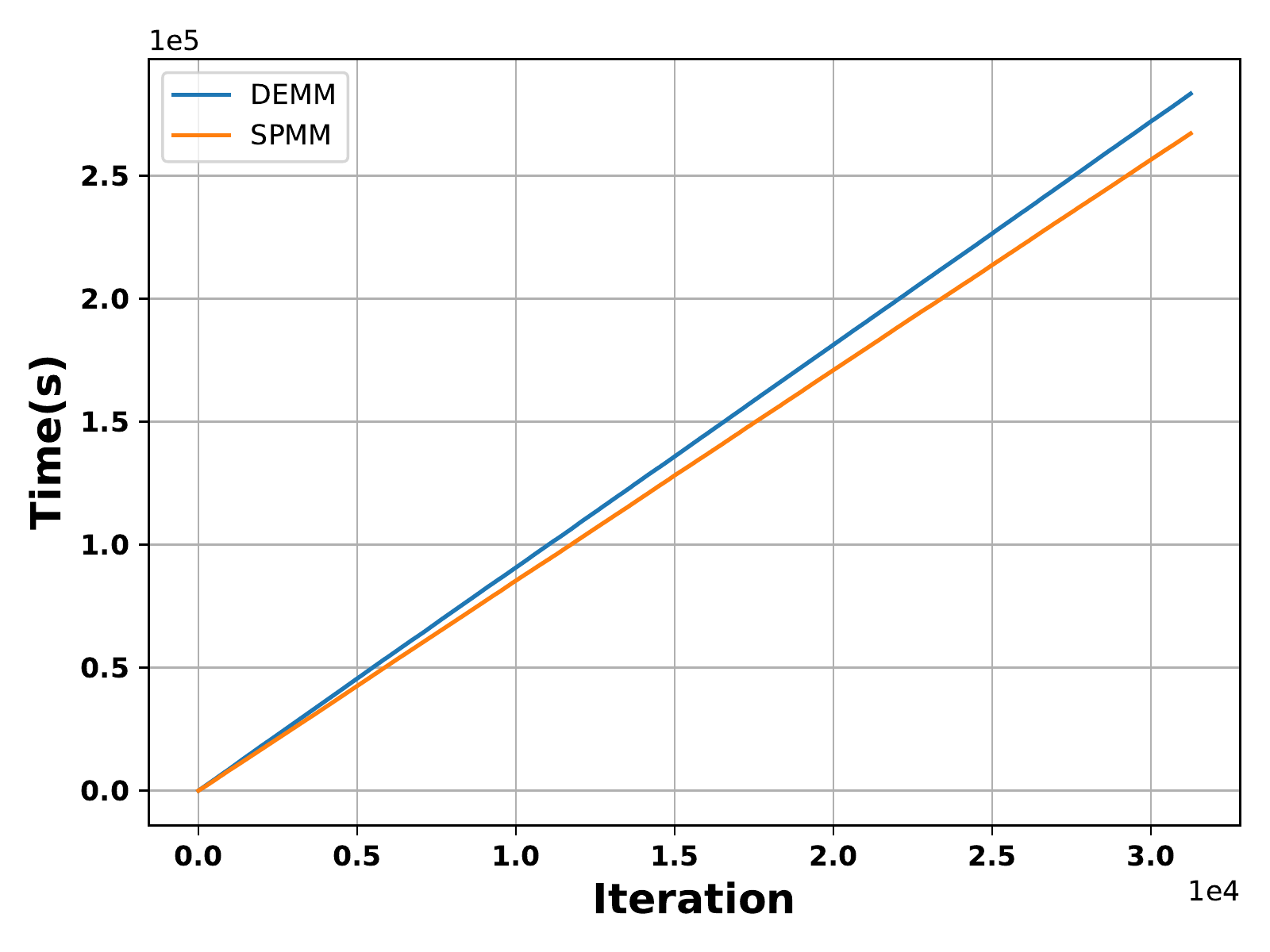}
    \vspace{-0.1cm}
    \caption{Running Time~\emph{v.s.} Training iteration.}
    \label[figure]{speedup}
    \vspace{-0.4cm}
\end{figure}

% \section{Limitation and Societal Impacts}

% \textbf{Limitation.} Our method Sparse SAM is mainly based on sparse operation. At present, the sparse operation that has been implemented is only 2:4 sparse operation. The 2:4 sparse operation requires that there are at most two non-zero values in four contiguous memory, which does not hold for us. To sum up, there is currently no concrete implemented sparse operation to achieve training acceleration. But in the future, with the development of hardware for sparse operation, our method has great potential to achieve truly training  acceleration.

% \textbf{Societal Impacts.} 
% In this paper, we provide a Sparse SAM algorithm that reduces computation burden and improves model generalization. In the future, we believe that with the development of deep learning, more and more models need the guarantee of generalization and also the efficient training. Different from the work on sparse networks, our proposed Sparse SAM does not compress the model for hardware limited device, but instead accelerates model training. It's helpful for individuals or laboratories which are lack computing resources.

\end{appendices}